\documentclass{article}
\usepackage{microtype}
\usepackage{graphicx}
\usepackage{subcaption}
\usepackage{booktabs} 
\usepackage{stmaryrd} 

\usepackage{hyperref}

\usepackage[accepted]{icml2024}

\usepackage{amsmath}
\usepackage{amssymb}
\usepackage{mathtools}
\usepackage{amsthm}

\usepackage[capitalize,noabbrev]{cleveref}

\usepackage{xspace}
\usepackage{tikzit, tikz-cd, circuitikz} 
\usepackage{quiver}
\usepackage{xparse}
\usepackage{scalefnt}
\usepackage[frozencache,cachedir=minted-cache]{minted}
\usepackage{etoolbox}
\usepackage{thmtools}
\usepackage{dirtytalk}

\makeatletter
\patchcmd{\minted@colorbg}{\medskip}{}{}{}
\patchcmd{\endminted@colorbg}{\medskip}{}{}{}
\makeatother

\newminted[code]{haskell}{
  bgcolor=white,
  frame=lines,
  style=manni,
  fontfamily=courier,
  linenos=false,
  escapeinside=@@,
}

\NewDocumentCommand{\scaletikzfig}{ O{1} O{1} m }{
  \begin{center}
     \tikzsetnextfilename{#3}
     \scalefont{#2}\scalebox{#1}{\input{#3.tikz}}
   \end{center}
}

\usetikzlibrary{decorations.markings} 
\usetikzlibrary{external}
\graphicspath{ {./figures/} {./figures/tikz} }

\tikzstyle{Medium box}=[fill=white, draw=black, shape=rectangle, tikzit shape=rectangle, minimum width=1.5cm, minimum height=1.5cm]
\tikzstyle{Large box}=[fill=white, draw=black, shape=rectangle, minimum width=5cm, minimum height=5cm]
\tikzstyle{small}=[fill=white, draw=black, shape=rectangle, minimum width=1cm, minimum height=0.5]
\tikzstyle{loss}=[fill=white, draw=black, shape=rectangle, minimum width=1cm, minimum height=3.5cm]
\tikzstyle{Medium}=[fill=white, draw=green, shape=rectangle, minimum width=1.5cm, minimum height=1.5cm, line width=0.8]
\tikzstyle{blue}=[fill=white, draw=blue, shape=rectangle, minimum width=1.5cm, minimum height=1.5cm]
\tikzstyle{red}=[fill=white, draw=red, shape=rectangle, minimum width=1.5cm, minimum height=1.5cm]
\tikzstyle{Long}=[fill=white, draw=black, shape=rectangle, minimum height=4.5cm, minimum width=1.5cm]
\tikzstyle{blackcircle}=[fill=black, draw=black, shape=circle, minimum width=0.2cm, inner sep=0pt]
\tikzstyle{Small box}=[fill=white, draw=black, shape=rectangle, minimum height=0.5cm, minimum width=0.5cm]
\tikzstyle{elongated}=[fill=white, draw=black, shape=rectangle, minimum height=1cm, minimum width=3cm]
\tikzstyle{circ}=[fill=white, draw=black, shape=circle]
\tikzstyle{blank}=[fill=white, draw=white, shape=circle]
\tikzstyle{none}=[fill=none, draw=none]
\tikzstyle{copy}=[fill=white, draw=black, shape=circle, minimum height=0.2cm, inner sep=0]
\tikzstyle{varCopy}=[fill=black, draw=black, shape=circle, minimum height=0.2cm, inner sep=0]
\tikzstyle{copy2}=[fill=black, draw=black, shape=circle, minimum height=0.2cm, inner sep=0]
\tikzstyle{1morph1}=[fill=white, draw=black, shape=rectangle, minimum width=1cm, minimum height=1cm]
\tikzstyle{1morph}=[fill=white, draw=black, shape=rectangle, minimum width=0.75cm, minimum height=0.75cm, inner sep=0.1cm]
\tikzstyle{2morph2}=[fill=white, draw=black, shape=rectangle, minimum width=1cm, minimum height=2cm]
\tikzstyle{2morph}=[fill=white, draw=black, shape=rectangle, minimum width=1cm, minimum height=1.25cm, inner sep=0.1cm]
\tikzstyle{nmorph}=[fill=white, draw=black, shape=rectangle, minimum height=6cm, minimum width=1cm, inner sep=0.1cm]
\tikzstyle{1state}=[fill=white, draw=black, regular polygon, regular polygon sides=3, minimum height=0.5cm, regular polygon rotate=-30]
\tikzstyle{dbox}=[fill=white, draw=black, dashed, shape=rectangle, minimum width=2cm, minimum height=1cm, inner sep=0.1cm]
\tikzstyle{vdbox}=[fill=white, draw=black, dashed, shape=rectangle, minimum width=2cm, minimum height=1.5cm, inner sep=0.1cm]
\tikzstyle{bigbox}=[fill=white, draw=black, dashed, shape=rectangle, minimum width=2cm, minimum height=4cm, inner sep=0.1cm]
\tikzstyle{2state}=[inner sep=0.05cm, fill=white, draw=black, isosceles triangle, minimum width=1.25cm, isosceles triangle apex angle=90, shape border rotate=180]
\tikzstyle{var2state}=[inner sep=0.05cm, fill=white, draw=black, isosceles triangle, minimum width=1.25cm, isosceles triangle apex angle=60, shape border rotate=180]
\tikzstyle{g2state}=[inner sep=0.05cm, fill=white, draw=black, isosceles triangle, minimum width=6cm, isosceles triangle apex angle=110, shape border rotate=180]
\tikzstyle{bigstate}=[inner sep=0.05cm, fill=white, draw=black, isosceles triangle, minimum width=3cm, isosceles triangle apex angle=110, shape border rotate=180]
\tikzstyle{bigeffect}=[inner sep=0.05cm, fill=white, draw=black, isosceles triangle, minimum width=3cm, isosceles triangle apex angle=110]
\tikzstyle{g2effect}=[inner sep=0.05cm, fill=white, draw=black, isosceles triangle, minimum width=6cm, isosceles triangle apex angle=110]
\tikzstyle{2effect}=[inner sep=0.05cm, fill=white, draw=black, isosceles triangle, minimum width=1.25cm, isosceles triangle apex angle=90]
\tikzstyle{b2effect}=[inner sep=0.05cm, fill=white, draw=black, isosceles triangle, minimum width=2cm, isosceles triangle apex angle=90]
\tikzstyle{node}=[fill=black, draw=black, shape=circle, scale=0.5]
\tikzstyle{GroundLeft}=[fill=white, draw=black, shape=tlground, rotate=-90]
\tikzstyle{GroundRight}=[fill=white, draw=black, shape=tlground, rotate=90]
\tikzstyle{GroundDown}=[fill=white, draw=black, shape=tlground]
\tikzstyle{GroundUp}=[fill=white, draw=black, shape=tlground, rotate=180]

\tikzstyle{Black arrow}=[->]
\tikzstyle{Red line}=[-, draw=red, line width=0.7]
\tikzstyle{Red arrow}=[draw=red, ->]
\tikzstyle{Gray line}=[-, draw={rgb,255: red,191; green,191; blue,191}, line width=0.8]
\tikzstyle{Gray arrow}=[->, draw={rgb,255: red,191; green,191; blue,191}]
\tikzstyle{Blue line}=[-, draw=blue]
\tikzstyle{Blue arrow}=[->, draw=blue]
\tikzstyle{bluearrow}=[->, fill=none, draw={rgb,255: red,29; green,206; blue,255}, thick]
\tikzstyle{midArrow}=[-, decoration={{markings,mark=at position .5 with {\arrow{>}}}}, postaction=decorate]
\tikzstyle{arrow}=[->]
\tikzstyle{pointy}=[->]
\tikzstyle{lightnone}=[-, draw={rgb,255: red,191; green,191; blue,191}]
\tikzstyle{Filled edge shape}=[-, fill=white, draw=black]
\tikzstyle{Thin Grey Arrow}=[->, draw={rgb,255: red,191; green,191; blue,191}, line width=0.3]
\tikzstyle{Thin Grey Line}=[-, draw={rgb,255: red,191; green,191; blue,191}, line width=0.3]
\tikzstyle{Black arrow crossing over}=[->, decoration={{crossing over}}]

\newcount\foreachcount

\makeatletter
\let\ea\expandafter
\def\foreachLetter#1#2#3{\foreachcount=#1
  \ea\loop\ea\ea\ea#3\@Alph\foreachcount
  \advance\foreachcount by 1
  \ifnum\foreachcount<#2\repeat}
\def\definecal#1{\ea\gdef\csname c#1\endcsname{\ensuremath{\mathcal{#1}}\xspace}}
\foreachLetter{1}{27}{\definecal}

\makeatletter
\let\ea\expandafter
\def\foreachLetter#1#2#3{\foreachcount=#1
  \ea\loop\ea\ea\ea#3\@Alph\foreachcount
  \advance\foreachcount by 1
  \ifnum\foreachcount<#2\repeat}
\def\definecal#1{\ea\gdef\csname b#1\endcsname{\ensuremath{\mathbf{#1}}\xspace}}
\foreachLetter{1}{27}{\definecal}

\theoremstyle{plain}
\newtheorem{theorem}{Theorem}[section]
\newtheorem{proposition}[theorem]{Proposition}
\newtheorem{lemma}[theorem]{Lemma}
\newtheorem{corollary}[theorem]{Corollary}
\theoremstyle{definition}
\newtheorem{definition}[theorem]{Definition}

\newtheorem{example}[theorem]{Example}
\newtheorem{conjecture}[theorem]{Conjecture}
\theoremstyle{remark}
\newtheorem{remark}[theorem]{Remark}

\newcommand{\newdef}[1]{\textbf{#1}}

\newcommand{\Z}{\mathbb{Z}}
\newcommand{\R}{\mathbb{R}}
\newcommand{\N}{\mathbb{N}}
\newcommand{\id}{\text{id}}
\newcommand{\inl}{\text{inl}}
\newcommand{\inr}{\text{inr}}

\newcommand{\recurrent}{\textsf{rcnt}}
\newcommand{\recursive}{\textsf{rcsv}}
\newcommand{\cell}{\textsf{cell}}
\newcommand{\foldingRecurrentCell}{\cell^{\recurrent}}
\newcommand{\unfoldingRecurrentCell}{\product{\cell_o}{\cell_n}}
\newcommand{\mooreCell}{\cell^{\textsf{Moore}}}
\newcommand{\mealyCell}{\cell^{\textsf{Mealy}}}
\newcommand{\recursiveCell}{\cell^{\recursive}}

\newcommand{\colim}{\underrightarrow{\mathsf{lim}}}

\newcommand{\NamedCat}[1]{\mathsf{#1}}

\newcommand{\Set}{\NamedCat{Set}}

\newcommand{\Cat}{\NamedCat{Cat}}

\newcommand{\Ob}[1]{\mathsf{Ob} (#1)}

\newcommand{\Smooth}{\NamedCat{Smooth}}
\newcommand{\Vect}{\NamedCat{Vect}}

\newcommand{\Endo}{\mathsf{Endo}}
\newcommand{\Pendo}{\mathsf{Pendo}}
\newcommand{\Mnd}{\mathsf{Mnd}}
\newcommand{\Cmnd}{\mathsf{Cmnd}}
\newcommand{\AlgEndo}[1]{\NamedCat{Alg}_{\Endo}({#1})}
\newcommand{\AlgPEndo}[1]{\NamedCat{Alg}_{\Pendo}({#1})}
\newcommand{\AlgMnd}[1]{\NamedCat{Alg}_{\Mnd}({#1})}
\newcommand{\AlgLaxMnd}[1]{\NamedCat{Lax}\dsh\NamedCat{Alg}_{\Mnd}({#1})}
\newcommand{\AlgLaxEndo}[1]{\NamedCat{Lax}\dsh\NamedCat{Alg}_{\Endo}({#1})}

\newcommand{\Para}{\NamedCat{Para}}

\newcommand{\Psh}{\NamedCat{Psh}}
\newcommand{\Sh}{\NamedCat{CPsh}}

\newcommand{\dsh}{\textbf{-}}
\newcommand{\Free}{\mathsf{Free}}
\newcommand{\Cofree}{\mathsf{Cofree}}
\newcommand{\Fix}{\mathsf{Fix}}
\newcommand{\product}[2]{\langle {#1} , {#2} \rangle}
\newcommand{\coproduct}[2]{[{#1} , {#2}]}

\newcommand{\List}[1]{\mathsf{List} (#1)}
\newcommand{\ListM}[2]{\mathsf{List}_{#1 + 1}(#2)}
\newcommand{\Tree}[1]{\mathsf{Tree} (#1)}
\newcommand{\Stream}[1]{\mathsf{Stream} (#1)}
\newcommand{\Moore}[2]{\mathsf{Moore}_{#1, #2}}
\newcommand{\Mealy}[2]{\mathsf{Mealy}_{#1, #2}}

\newcommand{\Nil}{\mathsf{Nil}}
\newcommand{\Cons}{\mathsf{Cons}}
\newcommand{\Leaf}{\mathsf{Leaf}}
\newcommand{\Node}{\mathsf{Node}}
\newcommand{\StreamOutput}{\mathsf{output}}
\newcommand{\StreamNext}{\mathsf{next}}

\newcommand{\act}{\blacktriangleright}

\icmltitlerunning{Categorical Deep Learning}

\begin{document}

\twocolumn[
\icmltitle{\emph{Position:}
Categorical Deep Learning is an Algebraic Theory of All Architectures}



\icmlsetsymbol{equal}{*}

\begin{icmlauthorlist}
\icmlauthor{Bruno Gavranovi\'{c}}{equal,sym,edi}
\icmlauthor{Paul Lessard}{equal,sym}
\icmlauthor{Andrew Dudzik}{equal,gdm}\\
\icmlauthor{Tamara von Glehn}{gdm}
\icmlauthor{Jo\~{a}o G.M. Ara\'{u}jo}{gdm}
\icmlauthor{Petar Veli\v{c}kovi\'{c}}{gdm,cam}
\end{icmlauthorlist}
\icmlaffiliation{gdm}{Google DeepMind}
\icmlaffiliation{sym}{Symbolica AI}
\icmlaffiliation{cam}{University of Cambridge}
\icmlaffiliation{edi}{University of Edinburgh}

\icmlcorrespondingauthor{Bruno Gavranovi\'{c}}{bruno@brunogavranovic.com}
\icmlcorrespondingauthor{Paul Lessard}{paul@symbolica.ai}
\icmlcorrespondingauthor{Andrew Dudzik}{adudzik@google.com}
\icmlcorrespondingauthor{Petar Veli\v{c}kovi\'{c}}{petarv@google.com}

\icmlkeywords{Machine Learning, ICML}

\vskip 0.3in
]



\printAffiliationsAndNotice{The title of the paper should be read as ``Categorical Deep Learning is an Algebraic \{Theory of All Architectures\}'', \emph{not} ``Categorical Deep Learning is an \{Algebraic Theory\} of All Architectures''. \icmlEqualContribution} 

\begin{abstract}

We present our position on the elusive quest for a general-purpose framework for specifying and studying deep learning architectures. Our opinion is that the key attempts made so far lack a coherent bridge between specifying \emph{constraints} which models must satisfy and specifying their \emph{implementations}. Focusing on building a such a bridge, we propose to apply category theory---precisely, the universal \emph{algebra of monads} valued in a 2-category of \emph{parametric maps}---as a single theory elegantly subsuming both of these flavours of neural network design. To defend our position, we show how this theory recovers constraints induced by geometric deep learning, as well as implementations of many architectures drawn from the diverse landscape of neural networks, such as RNNs. We also illustrate how the theory naturally encodes many standard constructs in computer science and automata theory.


\end{abstract}

\section{Introduction}

One of the most coveted aims of deep learning theory is to provide a guiding \emph{framework} from which all neural network architectures can be principally and usefully derived. Many elegant attempts have recently been made, offering frameworks to categorise or describe large swathes of deep learning architectures: \citet{cohen2019general,xu2019can,bronstein2021geometric,chami2022machine,papillon2023architectures,jogl2023expressivity,weiler2023EquivariantAndCoordinateIndependentCNNs} to name a few.

We observe that there are, typically, two broad ways in which deep learning practitioners describe models. Firstly, neural networks can be specified in a \emph{top-down} manner, wherein models are described by the \emph{constraints} they should satisfy (e.g. in order to respect the structure of the data they process). Alternatively, a \emph{bottom-up} approach describes models by their \emph{implementation}, i.e. the sequence of tensor operations required to perform their forward/backward pass.

\subsection{Our Opinion}

It is our \textbf{opinion} that ample effort has already been given to both the top-down and bottom-up approaches \emph{in isolation}, and that there hasn't been sufficiently expressive theory to address them both \emph{simultaneously}. \emph{If we want a \textbf{general} guiding framework for \textbf{all} of deep learning, this needs to change.} To substantiate our opinion, we survey a few ongoing efforts on both sides of the divide.

One of the most successful examples of the top-down framework is \emph{geometric deep learning} \citep[GDL]{bronstein2021geometric}, which uses a group- and representation-theoretic perspective to describe neural network layers via symmetry-preserving constraints. The actual realisations of such layers are derived by solving 
\emph{equivariance constraints}. 

GDL proved to be powerful: allowing, e.g., to cast \emph{convolutional layers} as an exact solution to linear translation equivariance in grids \citep{fukushima1983neocognitron,lecun1998gradient}, and \emph{message passing} and \emph{self-attention} as instances of permutation equivariant learning over graphs \citep{gilmer2017neural,vaswani2017attention}. It also naturally extends to exotic domains such as \emph{spheres} \citep{cohen2018spherical}, \emph{meshes} \citep{de2020gauge} and \emph{geometric graphs} \citep{fuchs2020se}. While this elegantly covers many architectures of practical interest, GDL also has inescapable constraints.



Firstly, usability of GDL principles to \emph{implement} architectures directly correlates with how easy it is to resolve equivariance constraints. While \texttt{PyG} \citep{fey2019fast}, \texttt{DGL} \citep{wang2019deep} and \texttt{Jraph} \citep{jraph2020github} have had success for permutation-equivariant models, and \texttt{e3nn} \citep{e3nn_paper} for $\mathrm{E}(3)$-equivariant models, it is hard to replicate such success for areas where it is not known how to resolve equivariance constraints.

Because of its focus on groups, GDL is only able to represent equivariance to symmetries, but not all operations we may wish neural networks to align to are invertible \citep{worrall2019deep} or fully compositional \citep{de2020natural}. This is not a small collection of operations either; if we'd like to align a model to an arbitrary \emph{algorithm} \citep{xu2019can}, it is fairly common for the target algorithm to irreversibly transform data, for example when performing any kind of a path-finding contraction \citep{dudzik2022graph}. Generally, in order to reason about alignment to constructs in \underline{\textbf{computer science}}, we must go beyond GDL.

On the other hand, bottom-up frameworks are most commonly embodied in \emph{automatic differentiation} packages, such as \texttt{TensorFlow} \citep{abadi2016tensorflow}, \texttt{PyTorch} \citep{paszke2019pytorch} and \texttt{JAX} \citep{bradbury2018jax}. These frameworks have become indispensable in the implementation of deep learning models at scale. Such packages often have grounding in \emph{functional programming}: perhaps \texttt{JAX} is the most direct example, as it is marketed as \emph{``composable function transformations''}, but such features permeate other deep learning frameworks as well. Treating neural networks as ``pure functions'' allows for rigorous analysis on their computational graph, allowing a degree of type- and shape-checking, as well as automatic tensor shape inference and fully automated backpropagation passes.

The issues, again, happen closer to the boundary between the two directions---specifying and controlling for \emph{constraint satisfaction} is not simple with tensor programming. Inferring general properties (\emph{semantics}) of a program from its implementation (\emph{syntax}) alone is a substantial challenge for all but the simplest programs, pointing to a need to model more abstract properties of \underline{\textbf{computer science}} than existing frameworks can offer directly. The similarity of the requirement on both sides leads us to our present position. 


\subsection{Our Position}

It is our \textbf{position} that constructing a guiding framework for all of deep learning, requires robustly \emph{bridging} the top-down and bottom-up approaches to neural network specification with a \emph{unifying mathematical theory}, and that the concepts for this bridging should be coming from \underline{\textbf{computer science}}. \emph{Moreover, such a framework must generalise both \textbf{group theory} and \textbf{functional programming}---and a natural candidate for achieving this is \textbf{category theory}.}


It is worth noting that ours is not the first approach to either (a) observe neural networks through the lens of computer science constructs \citep{baydin18}, (b) explore the connection between syntax and semantics in neural networks \citep{sonoda2023deep,sonoda2023joint,sonoda2024unified} (b) apply Category Theory to machine learning \citep{gavranovic2024github}. 

However, we are unaware of any prior work that tackles the connection of neural network architectures and the algebras of parametric maps, as we will do in this paper. Further, prior art in syntax-semantics connections either assumes that the operations are taking place in some topological space or that neural network architectures have a very specific form---our framework assumes neither. Lastly, prior papers exploring Category Theory and Machine Learning are fragmented, scarce, and not cohesive---our paper seeks to establish a common, unifying framework for how category theory can be applied to AI.

To defend our position, we will demonstrate a unified categorical framework that is expressive enough to rederive standard GDL concepts (invariance and equivariance), specify implementations of complex neural network building blocks (recurrent neural networks), as well as model other intricate deep learning concepts such as weight tying. 


\subsection{The Power of Category Theory}

To understand where we are going, we must first put the field of category theory in context.
Minimally, it may be conceived of as a battle-tested system of \emph{interfaces} that are learned once, and then reliably applied across scientific fields.
Originating in abstract mathematics, specifically algebraic topology, category theory has since proliferated, and been used to express ideas from numerous fields in an uniform manner, helping reveal their previously unknown shared aspects.
Other than modern pure mathematics, which it thoroughly permeates, these fields include \emph{systems theory} \citep{capucci_towards_2022, niu_polynomial_2023}, \emph{bayesian learning} \citep{braithwaite_compositional_2023,cho_disintegration_2019}, and \emph{information theory and probability} \citep{leinster_entropy_2021, bradley_entropy_2021,sturtz_categorical_2015, heunen_convenient_2017, perrone_markov_2022}.

This growth has resulted in a reliable set of mature theories and tools; from algebra, geometry, topology, combinatorics to recursion and dependent types, etc.\ all of them with a mutually compatible interface.
Recently category theory has started to be applied to machine learning, in \emph{automatic differentiation} \citep{vakar_chad_2022,alvarez-picallo_functorial_2021,gavranovic_space-time_2022,elliott_simple_2018}, \emph{topological data analysis} \citep{guss_characterizing_2018}, \emph{natural language processing} \citep{lewis_compositionality_2019}, \emph{causal inference} \citep{jacobs_causal_2019,cohen2022towards}, even producing an entire categorical picture of gradient-based learning -- from architectures to backprop -- in \citet{cruttwell_categorical_2022,fundamental_components}, with a more implementation-centric view in \citet{Nguyen_2022}, and important earlier work \citep{fong_backprop_2021}.




\subsubsection{Essential Concepts}

Before we begin, we recall three essential concepts in category theory, that will be necessary for following our exposition. First, we define a \emph{category}, an elegant axiomatisation of a compositional structure.

\begin{definition}[Category]
A category, $\mathcal{C}$, consists of a collection\footnote{The term ``collection'', rather than set, avoids Russell's paradox, as objects may themselves be sets. Categories that can be described with sets are known as \emph{small categories}.} of \emph{objects}, and a collection of \emph{morphisms} between pairs of objects, such that:
\begin{itemize}
    \item For each object $A\in\cC$, there is a unique \emph{identity} morphism $\id_A : A\rightarrow A$.
    \item For any two morphisms $f : A\rightarrow B$ and $g : B\rightarrow C$, there must exist a unique morphism which is their \emph{composition} $g\circ f : A\rightarrow C$ .
\end{itemize}
subject to the following conditions:
\begin{itemize}
    \item For any morphism $f : A\rightarrow B$, it holds that $\id_B\circ f = f\circ \id_A = f$.
    \item For any three composable morphisms $f : A\rightarrow B$, $g : B\rightarrow C$, $h : C\rightarrow D$, composition is \emph{associative}, i.e., $h\circ(g\circ f)=(h\circ g)\circ f$.
\end{itemize}

We denote by $\cC(A, B)$ the collection of all morphisms from $A\in\cC$ to $B\in\cC$.
\end{definition}

We provide a typical first example:
\begin{example}[The $\Set$ Category] $\Set$ is a category whose objects are \emph{sets}, and morphisms are \emph{functions} between them.
\end{example}

And another example, important for geometric DL:
\begin{example}[Groups and monoids as categories]
\label{ex:delooping}
A group, $G$, can be represented as a category, $\cB G$, with a single object ($G$), and morphisms $g: G\rightarrow G$ corresponding to elements $g\in G$, where composition is given by the group's binary operation. Note that $G$ is a group if and only if these morphisms are \emph{isomorphisms}, that is, for each $g : G\rightarrow G$ there exists $h : G\rightarrow G$ such that $h\circ g = g\circ h = \id_G$. More generally, we can identify one-object categories, whose morphisms are not necessarily invertible, with monoids.
\end{example}
The power of category theory starts to emerge when we allow different categories to \emph{interact}. Just as there are functions of sets and homomorphisms of groups, there is a more generic concept of \emph{structure preserving maps} between categories, called \emph{functors}.
\begin{definition}[Functor]
    Let $\cC$ and $\cD$ be two categories. Then, $F : \cC\rightarrow\cD$ is a functor between them, if it maps each object and morphism of $\cC$ to a corresponding one in $\cD$, and the following two conditions hold:
\begin{itemize}
    \item For any object $A\in\cC$, $F(\id_A) = \id_{F(A)}$.
    \item For any composable morphisms $f, g$ in $\cC$, $F(g\circ f) = F(g)\circ F(f)$.
\end{itemize}
An \emph{endofunctor} on $\cC$ is a functor $F : \cC\rightarrow\cC$.
\end{definition}
Just as a functor is an interaction between categories, a \emph{natural transformation} specifies an interaction between functors; this is the third and final concept we cover here.


\begin{definition}[Natural transformation]
Let $F : \cC\rightarrow\cD$ and $G : \cC\rightarrow\cD$ be two functors between categories $\cC$ and $\cD$. A natural transformation $\alpha : F\Rightarrow G$ consists of a choice, for every object $X\in\cC$, of a morphism $\alpha_X : F(X)\rightarrow G(X)$ in $\cD$ such that, for every morphism $f : X\rightarrow Y$ in $\cC$, it holds that $\alpha_Y\circ F(f) = G(f)\circ\alpha_X$.

The morphism $\alpha_X$ is called the component  of the natural transformation $\alpha$ at the object $X$.
\end{definition}

The components of a natural transformation assemble into ``naturality squares'', commutative diagrams:
\[\begin{tikzcd}[ampersand replacement=\&]
	{F(X)} \& {F(Y)} \\
	{G(X)} \& {G(Y)}
	\arrow["{F(f)}", from=1-1, to=1-2]
	\arrow["{\alpha_X}"', from=1-1, to=2-1]
	\arrow["{G(f)}"', from=2-1, to=2-2]
	\arrow["{\alpha_Y}", from=1-2, to=2-2]
\end{tikzcd}\]
where a diagram \emph{commutes} if, for any two objects, any two paths connecting them correspond to the same morphism.




\section{From Monad Algebras to Equivariance}

Having set up the essential concepts, we proceed on our quest to define a categorical framework which subsumes and generalises geometric deep learning \citep{bronstein2021geometric}.
First, we will define a powerful notion (\textbf{monad algebra homomorphism}) and demonstrate that the special case of monads induced by \emph{group actions} is sufficient to describe \emph{geometric deep learning}.
Generalising from monads and their algebras to arbitrary endofunctors and their algebras, we will find that our theory can express functions that process structured data from computer science (e.g. \emph{lists} and \emph{trees})\footnote{To the best of our knowledge, these ideas were first conjectured in \citet{olah_neural_2015} in the language of functional programming.} and behave in stateful ways like \emph{automata}.

\subsection{Monads and their Algebras}
\begin{definition}[Monad]
    Let $\cC$ be a category.
    A monad on $\cC$ is a triple $(M,\eta,\mu)$ where $M:\cC\to\cC$ is an endofunctor, and $\eta : \id_\cC\Rightarrow M$ and $\mu : M \circ M \Rightarrow M$ are natural transformations (where here  $\_ \circ \_$ is functor composition), making diagrams in \cref{def:monad_diagrams} commute.
\end{definition}

\begin{example}[Group action monad]
  \label{ex:group_action_monad}
  Let $G$ be a group.
  Then the triple $(G \times -, \eta, \mu)$ is a monad on $\Set$, where
  \begin{itemize}
      \item $G \times - : \Set \to \Set$ is an endofunctor mapping a set $X$ to the set $G \times X$;
      \item $\eta: \id_\Set \Rightarrow G \times - : \Set \rightarrow \Set$ whose component at a set $X$ is the function $x \mapsto (e,x)$ where $e$ is the identity element of the group $G$; and
      \item $\mu: G \times G \times - \Rightarrow G \times - : \Set \rightarrow \Set$ whose component at a set $X$ is the function $(g,h,x) \mapsto (gh,x)$ with the implicit multiplication that of the group $G$.
  \end{itemize}
\end{example}
Group action monads are formal theories of group actions, but they do not allow us to actually \emph{execute} them on data. This is what \emph{algebras} do.


\begin{definition}[Algebra for a monad]
  \label{def:algebra_for_monad}
  An algebra for a monad $(M, \eta, \mu)$ on a category $\cC$ is a pair $(A, a)$, where $A\in\cC$ is a \emph{carrier object} and $a : M(A) \to A$ is a morphism of $\cC$ (\emph{structure map}) making the following diagram commute:
  \[\begin{tikzcd}[ampersand replacement=\&]
      A \& {M(A)} \& {M(M(A))} \& {M(A)} \\
      \& A \& {M(A)} \& A
      \arrow["{\eta_A}", from=1-1, to=1-2]
      \arrow["a", from=1-2, to=2-2]
      \arrow[Rightarrow, no head, from=1-1, to=2-2]
      \arrow["a"', from=1-4, to=2-4]
      \arrow["a"', from=2-3, to=2-4]
      \arrow["{M(a)}", from=1-3, to=1-4]
      \arrow["{\mu_A}"', from=1-3, to=2-3]
  \end{tikzcd}\]
\end{definition}  

\begin{example}[Group actions]
    \label{ex:group_actions}
    Group actions for a group $G$ arise as algebras of the aforementioned group action monad $G \times -$.
    Consider the carrier $\R^{\Z_w \times \Z_h}$, thought of as data on a $w\times h$ grid, and any of the usual group actions on $\Z_w \times \Z_h$: translation, rotation, permutation, scaling, or reflections.
 
    Each of these group actions induce an algebra on the carrier set $\R^{\Z_w \times \Z_h}$. 
    For instance, the translation group $(\Z_w \times \Z_h, +, 0)$ induces the algebra
      \[
          \act : \Z_w \times \Z_h \times \R^{\Z_w \times \Z_h} \to \R^{\Z_w \times \Z_h}
      \]
      defined as $((i', j') \act x)(i, j) = x(i - i',j - j')$.
      Here $x$ represents the grid data, $i, j$ specific pixel locations, and $i', j'$ the translation vector.
      We also specifically mention the trivial action of any group $\pi_X : G \times X \to X$ by projection.
\end{example}



A monad algebra can capture a particular input or output for group equivariant neural networks (as its carrier). That being said, geometric deep learning concerns itself with \emph{linear equivariant layers} between these inputs and outputs. In order to be able to describe those, we need to establish the concept of a \emph{morphism of algebras} for a monad.

\begin{definition}[$M$-algebra homomorphism]
    \label{def:algebra_morphism}
    Let $(M, \mu, \eta)$ be a monad on $\cC$, and $(A, a)$ and $(B, b)$ be $M$-algebras.
    An \newdef{$M$-algebra homomorphism} $(A, a) \to (B, b)$ is a morphism $f : A \to B$ of $\cC$ s.t. the following commutes:
\[\begin{tikzcd}[ampersand replacement=\&]
  {M(A)} \& {M(B)} \\
  A \& B
  \arrow["a"', from=1-1, to=2-1]
  \arrow["b", from=1-2, to=2-2]
  \arrow["f"', from=2-1, to=2-2]
  \arrow["{M(f)}", from=1-1, to=1-2]
\end{tikzcd}\]
\end{definition}
We recover equivariant maps as morphisms of algebras.

\begin{example}[Equivariant maps]
  \label{eq:equivariance}
  Equivariant maps are group action monad algebra homomorphisms.
  Consider any action from \cref{ex:group_actions}.
  An endomorphism of such an action---that is, a $G$-algebra on $\R^{\Z_w \times \Z_h}$---is an endormorphism of $\R^{\Z_w \times \Z_h}$ which induces a commutative diagram 
\[\begin{tikzcd}[ampersand replacement=\&]
  {G \times \R^{\Z_w \times \Z_h}} \& {G \times \R^{\Z_w \times \Z_h}} \\
  {\R^{Z_w \times \Z_h}} \& {\R^{\Z_w \times \Z_h}}
  \arrow["\act"', from=1-1, to=2-1]
  \arrow["\act", from=1-2, to=2-2]
  \arrow["f"', from=2-1, to=2-2]
  \arrow["{G \times f}", from=1-1, to=1-2]
\end{tikzcd}\]
which, elementwise, unpacks to the equation 
\[
  f(g \act x) = g \act f(x)
\]
The translation example, for instance, recovers the equation
  $f(((i', j') \act x)(i, j)) = (i', j') \act f(x)(i, j)$
which reduces to the usual constraint: $f(x(i - i', j - j')) = f(x)(i - i', j - j')$.
\end{example}

The concept of \emph{invariance} --- a special case of equivariance --- is unpacked in the Appendix (\cref{ex:invariance}).
It's worth reflecting on the fact that we have just successfully derived the key aim of \textbf{geometric deep learning}: finding neural network layers that are \emph{monad algebra homomorphisms} of monads associated with \emph{group actions}! 

Indeed, the template illustrated in Example \ref{eq:equivariance} is sufficient to explain \emph{any} architectures which are explained by Geometric DL; should the reader wish to see concrete examples---deriving graph neural networks \citep{velivckovic2023everything}, Spherical CNNs \citep{cohen2018spherical} and G-CNNs \citep{cohen2016group}---they may be found in Appendix \ref{app:gdl_further}.

To concretely derive such layers from these constraints, we need to make concrete the category in which $f : A\rightarrow B$ lives. A standard choice is to use $\Vect$, a category where objects are finite-dimensional \emph{vector spaces} and morphisms are \emph{linear maps} between these spaces. In such a setting, morphisms can be specified as \emph{matrices}, and the equivariance condition places constraints on the matrix's entries, resulting in effects such as \emph{weight sharing} or \emph{weight tying}. We provide a detailed derivation for two examples on a two-pixel grid in \cref{app:gdl_examples}.


\begin{remark}
\label{rem:natural_graph_networks}

When our monad is of the form $M\times -$, with $M$ a monoid, algebras are equivalent to $M$-actions, i.e. functors $\cB M\to\Set$, where $\cB M$ is the one-object category given in Example \ref{ex:delooping}, and algebra morphisms are equivalent to natural transformations. So in this case, our definition of equivariance coincides with the functorial version given in \citet{de2020natural}. But the connection here is much deeper---for any monad, we can think of its algebras, which we can think of as the semantics of the monad, as functors on certain categories encoding the monad's syntax, such as \emph{Lawvere theories}. For example, \citet{dudzik2022graph} use the fact that functors on the category of finite polynomial diagrams\footnote{Equivalently, the category of finitary dependent polynomial functors.} encode the algebraic structure of commutative semirings. We give further details on this connection in Appendix \ref{app:lawvere_theories}.
\end{remark}

\subsection{Endofunctors and their (Co)algebras}
\label{subsec:endofunctors_algebras}

Geometric deep learning, while elegant, is fundamentally constrained by the axioms of group theory. Monads and their algebras, however, are naturally generalised beyond group actions.
Here we show how, by studying (co)algebras of arbitrary \emph{endofunctors}, we can rediscover standard computer science constructs like lists, trees and automata. This rediscovery is not merely a passing observation; in fact, the endofunctor view of lists and trees turns out to naturally map to \emph{implementations} of neural architectures such as recurrent and recursive neural networks; see Appendix \ref{sec:parametric_endofunctor_algebras}. 

Then, in the next section we'll show how these more minimal structures, endofunctors and their algebras, may be augmented into the more structured notions of monads and their algebras.


\begin{definition}[Algebra for an endofunctor]
    \label{def:algebra_for_endofunctor}
    Let $\cC$ be a category and $F: \cC\rightarrow\cC$ an endofunctor on $\cC$. An algebra for $F$ is a pair $(A, a)$ where $A$ is an object of $\cC$ and $a : F(A) \to A$ is a morphism of $\cC$.
\end{definition}

Note that, compared to \cref{def:algebra_for_monad}, there are no equations this time;  $F$ is not equipped with any extra structure with which the structure map of an algebra could be compatible. Examples of endofunctor algebras abound \cite{jacobs_introduction_2016}, many of which are familiar to computer scientists.

\begin{example}[Lists]
  \label{ex:list_as_algebra}
  Let $A$ be a set, and consider the endofunctor $1 + A \times - : \Set \to \Set$.
  The set $\List{A}$ of lists of elements of type $A$ together with the map $\coproduct{\Nil}{\Cons} : 1 + A \times \List{A} \to \List{A}$ forms an algebra of this endofunctor.\footnote{$\coproduct{f}{g} : A + B \to C$ is notation for maps out of a coproduct; where $f : A \to C$ and $g : B \to C$}:
  Here $\Nil$ and $\Cons$ are two constructors for lists, allowing us to represent lists as the following datatype:
  \begin{code}
data List a = Nil
            | Cons a (List a)
  \end{code}
  It describes $\List{A}$ \emph{inductively}, as being formed either out of the empty list, or an element of type $A$ and another list.
  In \cref{fig:cells_as_algebras} we will see how this relates to \emph{folding} RNNs.
\end{example}

\begin{example}[Binary trees]
  \label{ex:tree_as_algebra}
  Let $A$ be a set.
  Consider the endofunctor $A + (-)^2 : \Set \to \Set$.
  The set $\Tree{A}$ of binary trees with $A$-labelled leaves, together with the map $\coproduct{\Leaf}{\Node} : A + \Tree{A}^2 \to \Tree{A}$ forms an algebra of this endofunctor.
  Here $\Leaf$ and $\Node$ are constructors for binary trees, enabling the following datatype representation:
\begin{code}
data Tree a = Leaf a
            | Node (Tree a) (Tree a)
\end{code}
It describes $\Tree{A}$ inductively, as being formed either out of a single $A$-labelled leaf or two subtrees.\footnote{This framework can also model any variations, e.g., $n$-ary trees with $A$-labelled leaves as algebras of $A + \List{-}$, or binary trees with $A$-labelled nodes as algebras of $1 + A \times (-)^2$.}
In \cref{fig:cells_as_algebras} we will relate this to recursive neural networks.
\end{example}
Dually, we also study \emph{coalgebras} for an endofunctor (where the structure morphism $a : A\rightarrow F(A)$ points the other way (\cref{def:coalgebra_for_endofunctor}).
Intuitively, while algebras offer us a way to model computation guaranteed to terminate, coalgebras offer us a way to model potentially infinite computation.
They capture the semantics of programs whose guarantee is not termination, bur rather \emph{productivity} \citep{atkey_productive_2013}, and as such are excellent for describing servers, operating systems, and automata \cite{rutten_universal_2000,jacobs_introduction_2016}.
We will use endofunctor coalgebras to describe one such automaton---the Mealy machine \cite{mealy_method_1955}.

\begin{example}[Mealy machines]
  \label{ex:mealy_machine_as_coalgebra}
  Let $O$ and $I$ be sets of possible \emph{outputs} and \emph{inputs}, respectively.
  Consider the endofunctor $(I \to O \times -) : \Set \to \Set$.
  Then the set $\Mealy{O}{I}$ of Mealy machines with outputs in $O$ and inputs in $I$, together with the map $\StreamNext : \Mealy{O}{I} \to (I \to O \times \Mealy{O}{I})$ is a coalgebra of this endofunctor.
\begin{code}
data Mealy o i = MkMealy {
  next :: i -> (o, Mealy o i)
}
\end{code}
This describes Mealy machines \emph{coinductively}, as \emph{systems} which, given an input, produce an output and another Mealy machine. 
In \cref{fig:cells_as_algebras} we will relate this to full recurrent neural networks, and, in \cref{ex:stream_as_coalgebra,ex:moore_machine_as_coalgebra} we coalgebraically express two other fundamental classes of automata: streams and Moore machines.
\end{example}

We have expressed data structures and automata using (co)algebras for an endofunctor. Just as in the case of GDL, in order to describe (linear) \emph{layers of neural networks} between them, we need to establish the concept of a \emph{homomorphism of endofunctor (co)algebras}. The definition of a homomorphism of algebras for an endofunctor mirrors\footnote{Because it does not rely on the extra structure monads have.}  \cref{def:algebra_morphism}, while the definition of a homomorphism of coalgebras has the structure maps pointing the other way.
\begin{example}[Folds over lists as algebra homomorphisms]
  \label{ex:lists_initial}
  Consider the endofunctor $(1 + A \times -)$ from \cref{ex:list_as_algebra}, and an algebra homomorphism from $(\List{A}, \coproduct{\Nil}{\Cons})$ to any other $(1 + A \times -)$-algebra $(X, \coproduct{r_0}{r_1})$:
\[\begin{tikzcd}[ampersand replacement=\&]
	{1 + A \times \List{A}} \&\& {1 + A \times X} \\
	{\List{A}} \&\& X
	\arrow["{\coproduct{\Nil}{\Cons}}"', from=1-1, to=2-1]
	\arrow["{f_r}"', from=2-1, to=2-3]
	\arrow["{\coproduct{r_0}{r_1}}", from=1-3, to=2-3]
	\arrow["{1 + A \times f_r}", from=1-1, to=1-3]
\end{tikzcd}\]
Then the map $f_r : \List{A}\to X$ is, necessarily, a \emph{fold} over a list, a concept from functional programming which describes how a single value is obtained by operating over a list of values.
It is implemented by recursion on the input:
  \begin{code}
@f\_r@ :: List a -> x
@f\_r@ Nil = @r\_0@ ()
@f\_r@ (Cons h t) = @r\_1@ h (@f\_r@ t)
  \end{code}
This recursion is \emph{structural} in nature, meaning it satisfies the following two equations which arise by unpacking the algebra homomorphism equations elementwise:
  \begin{align}
    f_r(\Nil) &=  r_0(\bullet) \label{eq:list_equation1} \\
    f_r(\Cons(h, t)) &= r_1(h, f_r(t)) \label{eq:list_equation2}
  \end{align}
  \Cref{eq:list_equation1} tells us that we get the same result if we apply $f_r$ to $\Nil$ or apply $r_0$ to the unique element of the singleton set\footnote{Where \mintinline{haskell}{()} was used to denote it in Haskell notation.}. \cref{eq:list_equation2} tells us that, starting with the head and tail of a list, we get the same result if we concatenate the head to the tail, and then process the entire list with $f_r$, or if we process the tail first with $f_r$, and then combine the result with the head using $r_1$.
\end{example}
It is important to remark that these equations \emph{\textbf{generalise} equivariance constraints} over a list structure. Both group equivariance and \cref{eq:list_equation1,eq:list_equation2} intuitively specify a function that is predictably affected by certain operations---but for the case of lists, these operations (concatenating) are \emph{not} group actions, as attaching an element to the front of the list does \emph{not} leave the list unchanged.

\begin{remark}\label{rmk:freemonoid}
Interestingly, given an algebra $(X, \coproduct{r_0}{r_1})$, there can only ever be \emph{one} algebra homomorphism from lists to it!
This is because $(\List{A}, \coproduct{\Nil}{\Cons})$ is an \emph{initial object} (\cref{def:initial_object}) in the category of $(1 + A \times -)$-algebras.
The fact that these are initial arises from a deeper fact related to the fact that, in many cases, for a given endofunctor there is a monad whose category of monad algebras is equivalent to the original category of endofunctor algebras.
We note this because the construction which takes us from one to the other, the so-called \emph{algebraically free monad} on an endofunctor, will be seen in Part 3 to derive RNNs and other similar architectures from first principles.
\end{remark}

\begin{example}[Tree folds as algebra homomorphisms]
  \label{ex:trees_initial}
  Consider the endofunctor $A + (-)^2$ from \cref{ex:tree_as_algebra}, and an algebra homomorphism from $(\Tree{A}, \coproduct{\Leaf}{\Node})$ to any other $(A + (-)^2)$-algebra $(X, \coproduct{r_0}{r_1})$:
\[\begin{tikzcd}[ampersand replacement=\&]
	{A + \Tree{A}^2} \&\& {A + X^2} \\
	{\Tree{A}^2} \&\& X
	\arrow["{\coproduct{\Leaf}{\Node}}"', from=1-1, to=2-1]
	\arrow["{f_r}"', from=2-1, to=2-3]
	\arrow["{\coproduct{r_0}{r_1}}", from=1-3, to=2-3]
	\arrow["{A + f_r^2}", from=1-1, to=1-3]
\end{tikzcd}\]
Then the map $f_r$ is necessarily a \emph{fold} over a tree. As with lists, it is implemented by recursion on the input, which is \emph{structural} in nature:
\begin{code}
@f\_r@ :: Tree a -> x
@f\_r@ (Leaf a) = @r\_0@ a
@f\_r@ (Node l r) = @r\_1@ (@f\_r@ l) (@f\_r@ r)
\end{code}
This means that it satisfies the following two equations which arise by unpacking the algebra homomorphism equations elementwise:
\begin{align}
  f_r(\Leaf(a)) &=  r_0(a) \label{eq:tree_equation1} \\
  f_r(\Node(l, r)) &= r_1(f_r(l), f_r(r)) \label{eq:tree_equation2}
\end{align}
These can also be thought as describing generalised equivariance over binary trees, analogously to lists. 
\end{example}

Dual to algebra homomorphisms and folds over inductive data structures, \emph{coalgebra homomorphisms} are categorical semantics of \emph{unfolds} over coinductive data structures.

\begin{example}[Unfolds as coalgebra homomorphisms]
  \label{ex:mealy_terminal}
  Consider the endofunctor $(I \to O \times -)$ from \cref{ex:mealy_machine_as_coalgebra}, and a coalgebra homomorphism from $(\Mealy{O}{I}, \StreamNext)$ to any other $(I \to O \times -)$-coalgebra $(X, n)$:
\[\begin{tikzcd}[ampersand replacement=\&]
	X \&\& {\Mealy{O}{I}} \\
	{(I \to O \times X)} \&\& {(I \to O \times \Mealy{O}{I})}
	\arrow["n"', from=1-1, to=2-1]
	\arrow["\StreamNext", from=1-3, to=2-3]
	\arrow["{f_n}", from=1-1, to=1-3]
	\arrow["{(I \to O \times f_n)}"', from=2-1, to=2-3]
\end{tikzcd}\]
The map $f_n$ here can be thought of as a generalised \emph{unfold}, a concept from functional programming describing how a potentially infinite data structure is obtained from a single value.
It is implemented by a \emph{corecursive} function:
\begin{code}
@f\_n@ :: x -> Mealy o i
@f\_n@ x = MkMealy
   \i -> let (o', x') = n x i
          in (o', @f\_n@ x')
\end{code}
which is again structural in nature.
This means that it satisfies the following two equations which arise by unpacking the coalgebra homomorphism equations elementwise:
\begin{align}
  n(x)(i)_1 = \StreamNext(f_n(x))(i)_1 \label{eq:mealy_equation1}\\
  f_n(n(x)(i)_2) = \StreamNext(f_n(x))(i)_2 \label{eq:mealy_equation2}
\end{align}
\Cref{eq:mealy_equation1} tells us that the output of the Mealy machine produced by $f_n$ at state $x$ and input $i$ is given by the output of $n$ at state $x$ and input $i$, and \cref{eq:mealy_equation2} tells us that the next Mealy machine produced at $x$ and $i$ is the one produced by $f_n$ at $n(x)_2(i)$.
\end{example}

This, too, \emph{\textbf{generalises equivariance}} constraints, now describing an interactive automaton which is by no means invertible.
Instead, it is \emph{dynamic} in nature, producing outputs which are dependent on the current state of the machine and previously unknown inputs.
Lastly, in \cref{ex:streams_terminal,ex:moore_terminal}, we show that two kinds of automata---streams and Moore machines---are also examples of coalgebra homomorphisms.
Just as before, we can embed all of our objects and morphisms into $\Vect$ to study the weight sharing constraints induced by such a condition---see \cref{ex:streams_terminal}.

\subsection{Where to Next?}

\begin{figure*}
  \begin{equation*}
    \begin{tikzcd}[ampersand replacement=\&]
    	{\substack{\text{Folding recurrent} \\ \text{neural network}}} \&\& {\substack{\text{Unfolding recurrent} \\ \text{neural network}}} \&\& {\substack{\text{Recursive} \\ \text{neural network}}} \&\& {\substack{\text{Full recurrent} \\ \text{neural network}}} \& {\substack{\text{``Moore machine''} \\ \text{neural network}}} \\
    	{1 + A \times S} \&\& S \&\& {A + S^2} \&\& S \& S \\
    	S \&\& {O \times S} \&\& S \&\& {(I \to O \times S)} \& {O \times (I \to S)}
    	\arrow["{(P, \foldingRecurrentCell)}", from=2-1, to=3-1]
    	\arrow["{(P, \unfoldingRecurrentCell)}", from=2-3, to=3-3]
    	\arrow["{(P, \recursiveCell)}", from=2-5, to=3-5]
    	\arrow["{(P, \mealyCell)}", from=2-7, to=3-7]
    	\arrow["{(P, \mooreCell)}", from=2-8, to=3-8]
    \end{tikzcd}
\end{equation*}
  \begin{minipage}{0.18\textwidth}
    \scaletikzfig[0.55]{folding_rnn_cell}
  \end{minipage}
  \hfill
  \begin{minipage}{0.18\textwidth}
    \scaletikzfig[0.55]{unfolding_recurrent_cell}
  \end{minipage}
  \hfill
  \begin{minipage}{0.18\textwidth}
    \scaletikzfig[0.55]{recursive_cell_lab-leaves_no-nodes}
  \end{minipage}
  \hfill
  \begin{minipage}{0.18\textwidth}
    \scaletikzfig[0.55]{mealy_cell}
  \end{minipage}
  \hfill
  \begin{minipage}{0.18\textwidth}
    \scaletikzfig[0.55]{moore_cell2}
  \end{minipage}
  \caption{Parametric (co)algebras provide a high-level framework for describing structured computation in neural networks.}
  \label{fig:cells_as_algebras}
\end{figure*}
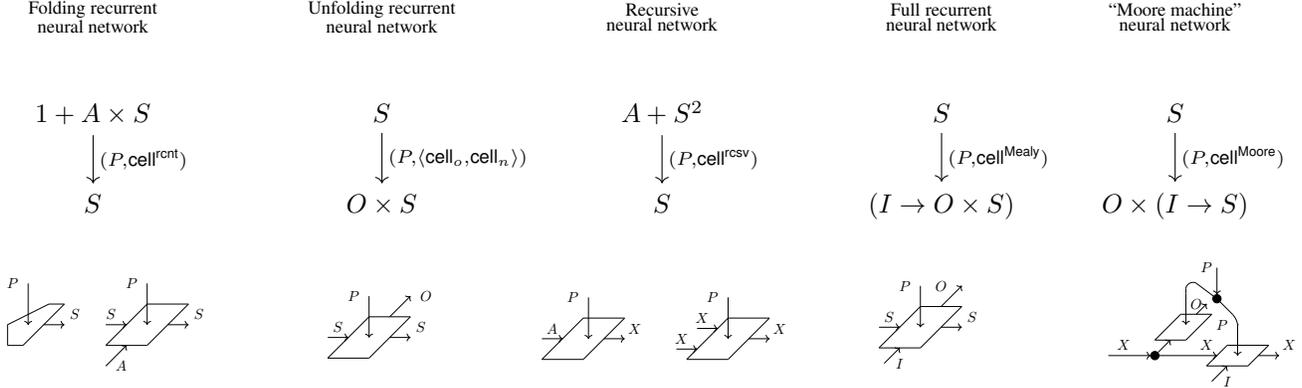

Let's take a step back and understand what we've done.
We have shown that an existing categorical framework \emph{uniformly} captures a number of different data structures and automata, as particular (co)algebras of an endofunctor.
By choosing a well-understood data structure, we induce a structural constraint on the control flow of the corresponding neural network, by utilising homomorphisms of these endofunctor (co)algebras.
These follow the same recipe as monad algebra homomorphisms, and hence can be thought of as \emph{generalising equivariance}---describing functions \emph{beyond} what geometric deep learning can offer.

This is concrete evidence for our position---that \emph{\textbf{categorical algebra homomorphisms} are suitable for capturing various constraints one can place on deep learning architectures}. Our evidence so far rested on \emph{endofunctor algebras}, which are a particularly fruitful variant.

However, this construct leaves much to be desired. One major issue is that, to prescribe any notion of \emph{weight sharing}, for all of these examples we have implicitly assumed homomorphisms to be linear transformations by placing them into the category $\Vect$.
But most neural networks aren't simply linear maps, meaning that these analyses are limited to analysing their individual layers.
In the standard example of recurrent neural networks, an RNN cell is an \emph{arbitrary} differentiable function, usually composed out of a sequence of linear and non-linear maps.
Further, in both practice and theory, neural networks are treated as \emph{parametric} functions.
Specifically, in frameworks like \texttt{JAX}, the parameters often need to be passed \emph{explicitly} to any forward or backward pass of a neural network.

How can we explicitly model parameters and non-linear maps, without abandoning the presented categorical framework?
Furthermore, in practice---with recurrent or recursive \citep{socher_recursive_2013} neural networks, for instance---there are established techniques for weight sharing.
Can we establish formal criteria for when these techniques are \emph{correct}? Just as how we generalised GDL to the setting of category theory, we can go further: to the setting of \emph{2-categories}---the setting we use to study \emph{parametric morphisms}.

\section{2-Categories and Parametric Morphisms}
\label{part:parametric_morphisms}

While category theory is a powerful framework, it leaves much to be desired in terms of \emph{higher-order relationships} between morphisms.
It only deals with \emph{sets} of morphisms, with no possible way to \emph{compare} elements of these sets.
This is where the theory of \emph{2-categories} comes in, which deals with an entire \emph{category} of morphisms.
While in a (1-)category, one has objects and morphisms between objects, in a 2-category one has objects (known as \emph{0-morphisms}), morphisms between objects (\emph{1-morphisms}), and morphisms between \emph{morphisms} (\emph{2-morphisms}).
We have, in fact, already secretly seen an instance of a 2-category, $\Cat$, when defining the essential concepts of category theory. Specifically, in $\Cat$, objects are \emph{categories}, morphisms are \emph{functors} between them, and 2-morphisms are \emph{natural transformations} between functors.

\subsection{The 2-category $\Para$}

In this section we define an established 2-category $\Para$ \cite{cruttwell_categorical_2022,capucci_towards_2022}, and proceed to unpack the manner which we posit weight sharing can be modelled formally in it.

While it shares objects with the category $\Set$, its 1-morphisms are not functions, but \emph{parametric} functions.
That is, a 1-morphism $A \to B$ here consists of a pair $(P, f)$, where $P \in \Set$ and $f : P \times A \to B$.

\begin{minipage}{0.6\columnwidth}
$\Para$ morphisms admit an elegant graphical formalism. Parameters ($P$) are drawn \emph{vertically}, signifying that they are part of the \emph{morphism}, and not objects.
\end{minipage}
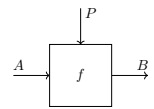
\hfill
\begin{minipage}{0.3\columnwidth}
\begin{figure}[H]
  \scaletikzfig[0.55]{para}
  \label{fig:para_morphism}
\end{figure}
\end{minipage}

The 2-category $\Para$ models the algebra of composition of neural networks; the sequential composition of parametric morphisms composes the parameter spaces in parallel (\cref{fig:para_composition}).

The 2-morphisms in $\Para$ capture \emph{reparameterisations} between parametric functions.
Importantly, this allows for the \textbf{explicit treatment of weight tying}, where a parametric morphism $(P \times P, f)$ can have its weights tied by precomposing with the copy map $\Delta_P : P \to P \times P$.
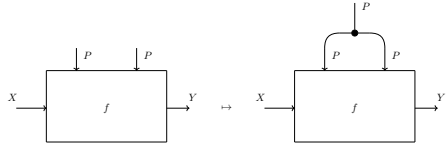
\begin{figure}[H]
  \scaletikzfig[0.4]{weight_tying}
  \label{fig:weight_tying}
\end{figure}
This 2-category\footnote{More precisely, the construction $\Para(\_)$. See \cref{sec:2-categories_para(c)}.}  is one of the key components in the categorical picture of gradient-based learning \cite{cruttwell_categorical_2022}. But we hypothesise that more is true (Appendix \ref{sec:parametric_endofunctor_algebras}):
\begin{center}
It is our position that the 2-category $\Para$ and 2-categorical algebra valued in it provide a \emph{\textbf{formal theory of neural network architectures}, establish \emph{\textbf{formal criteria}} for weight tying correctness and \emph{\textbf{inform design}} of new architectures.}
\end{center}

\subsection{2-dimensional Categorical Algebra}
2-category theory is markedly richer than 1-category theory. 

While diagrams in a 1-category either commute or do not commute, in a 2-category, they serve as a \emph{1-skeleton} to which 2-morphisms \emph{attach}. In \emph{any} 2-category a square may: commute, pseudo-commute, lax-commute, or oplax\footnote{Often also called \emph{colax}.} commute, meaning, respectively, that relevant paths paths are equal, isomorphic, or there is a 2-morphism from one to the other in one direction or the other. The diagrams below present these four options, with 2-morphisms denoted by double arrows.
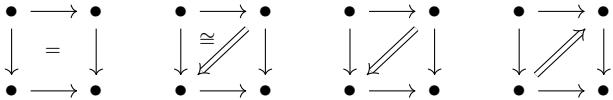
\[\begin{tikzcd}[ampersand replacement=\&,column sep=scriptsize]
	\bullet \& \bullet \& \bullet \& \bullet \& \bullet \& \bullet \& \bullet \& \bullet \\
	\bullet \& \bullet \& \bullet \& \bullet \& \bullet \& \bullet \& \bullet \& \bullet
	\arrow[from=1-1, to=1-2]
	\arrow[from=1-2, to=2-2]
	\arrow[from=1-1, to=2-1]
	\arrow[from=2-1, to=2-2]
	\arrow["{=}"{description}, draw=none, from=2-1, to=1-2]
	\arrow[from=1-3, to=2-3]
	\arrow[from=2-3, to=2-4]
	\arrow[from=1-3, to=1-4]
	\arrow[from=1-4, to=2-4]
	\arrow["\cong"', Rightarrow, from=1-4, to=2-3]
	\arrow[from=1-5, to=2-5]
	\arrow[from=1-5, to=1-6]
	\arrow[from=1-6, to=2-6]
	\arrow[from=2-5, to=2-6]
	\arrow[Rightarrow, from=1-6, to=2-5]
	\arrow[from=1-7, to=2-7]
	\arrow[from=1-7, to=1-8]
	\arrow[from=1-8, to=2-8]
	\arrow[from=2-7, to=2-8]
	\arrow[Rightarrow, from=2-7, to=1-8]
\end{tikzcd}\]
In the long run, we expect that \emph{all} of these notions will apply to, either explaining or specifying, aspects of neural architecture past, present and future. Focusing on just one of them, the \emph{lax algebras} are sufficient to derive \emph{recursive}, \emph{recurrent}, and similar neural networks from first principles. Notably, morphisms of lax algebras are also expressive enough to capture \emph{1-cocycles}, used to formalise asynchronous neural networks in \cite{dudzik2023asynchronous}---see Appendix \ref{app:gdl_examples}.

Interestingly, this story of how an individual recurrent, recursive, etc. neural network cell generates a full recursive, recurrent etc. neural network is a particular 2-categorical analogue to the story of \emph{algebraically free monads} on an endofunctor we briefly mentioned in \cref{rmk:freemonoid}.

For all the examples of endofunctors in \cref{subsec:endofunctors_algebras}, there is a monad whose category of algebras $\Free_{\Mnd}(F)$ is equivalent to the category of algebras for the original endofunctor $F$.
We obtain $\Free_{\Mnd}(F)$ by iterating $F$ until it \emph{stabilises}, meaning further application of the endofunctor does not change the composition. Functional programmers may recognise this from the implementation of free monads in Haskell, while formally this is defined using \emph{colimits} (see \cref{sec:Kellys_transfinite_construction}). Using this concept, we can define a functor mapping an $F$-algebra $\left( A,a \right)$ to the $\Free_{\Mnd}(F)$-algebra $(A,\underrightarrow{\mathsf{lim}}(a\circ Fa \circ F^2 a \circ \cdots \circ F^n a))$, connecting appropriate endofunctor algebras to monad algebras.

But in the 2-dimensional case, we study 
the relationship between $\AlgLaxEndo{F}$ and $\AlgLaxMnd{\Free_{\Mnd}(F)}$ and need contend not only with generating the 1-dimensional structure map, but also the 2-cells of the lax algebra for a monad.

To reconcile this with concrete applications, we note that we do not need to study general 2-endofunctors and 2-monads on $\Para$. Rather, examples which concern us arise from specific 1-categorical algebras (group action monads, inductive types, etc.), which are \emph{augmented} into 2-monads on $\Para$. As we prove in \cref{thm:parameters_lax_algebra_comonoids}, the lax cells of such algebras are actually \emph{comonoids}. The fact that we can duplicate or delete entries in vectors---the essence of tying weights---is the informal face of this comonoid structure.


We can now describe, even if space constraints prevent us from adequate level of detail, the universal properties of recurrent, recursive, and similar models: they are \emph{lax algebras} for \emph{free parametric monads} generated by \emph{parametric endofunctors}! Having lifted the concept of algebra introduced in Part 2 into 2-categories, we can now describe several influential neural networks \emph{fully} (not just their individual layers!) from first principles of functional programming.

\section{New Horizons}

Our framework gives \emph{the correct definition} of numerous variants of structured networks as universal parametric counterparts of known notions in computer science.
This immediately opens up innumerable avenues for research.

Firstly, any results of categorical deep learning as presented here rely on \emph{choosing} the right category to operate in; much like results in geometric deep learning relied on the choice of symmetry group. However, we have seen that monad algebras---which \emph{generalise} equivariance constraints---can be parametric, and \emph{lax}.
As a consequence, the kinds of equivariance constraints we can \emph{learn} become more general: we hypothesise neural networks that can learn not merely conservation laws (as in \citet{alet2021noether}), but verifiably correct logical argument, or code. This has ramifications for \emph{code synthesis}: we can, for example, specify neural networks  that learn only \emph{well-typed functions} by choosing appropriate algebras as their domain and codomain. 

This is made possible by our framework's generality: for example, by choosing \emph{polynomial functors} as endofunctors we get access to \emph{containers} \citep{abbott_categories_2003,hutchison_higher-order_2010}, a uniform way to program with and reason about datatypes and polymorphic functions.
By combining these insights with recent advances enabling \emph{purely functional} differentiation through inductive and coinductive types \citep{nunes_chad_2023}, we open new vistas for \emph{type-safe} design and implementation of neural networks in functional languages.

One major limitation of geometric deep learning was that it was typically only able to deal with individual neural network layers, owing to its focus on \emph{linear} equivariant functions (see e.g. \citet{maron2018invariant} for the case of graphs). All nonlinear behaviours can usually be obtained through composition of such layers with nonlinearities, but GDL typically makes no attempt to explain the significance of the choice of nonlinearity---which is known to often be a significant decision \citep{shazeer2020glu}. Within our framework, we can reason about architectural blocks spanning multiple layers---as evidenced by our weight tying examples---and hence we believe CDL should enable us to have a theory of architectures which properly treats nonlinearities.

Our framework also offers a \emph{proactive} path towards \emph{equitable} AI systems.
GDL already enables the architectural imposition of protected classes invariance (see \citet{Choraria2021BalancingFA} for example). This deals, at least partially, both with issues of inequity in training data and inequity in algorithms since such an invariant model is, by construction, exclusively capable of inference on the dimensions of latent representation which are orthogonal to protected class.

With CDL, we hope to enable even finer grained control. By way of categorical logic, we hope that CDL will lead us to a new and deeper understanding of the relationship between architecture and logic, in particular clarifying the logics of inductive bias. We hope that our framework will eventually allow us to specify the kinds of arguments the neural networks can use to come to their conclusions. This is a level of expressivity permitting reliable use for assessing bias, fairness in the reasoning done by AI models deployed at scale.
We thus believe that this is the right path to AI compliance and safety, and not merely explainable, but verifiable AI.



\section*{Impact Statement}

This paper presents work whose goal is to advance the field of Machine Learning. There are many potential societal consequences of our work, none which we feel must be specifically highlighted here.

\section*{Acknowledgments}

The authors wish to thank Razvan Pascanu and Yee Whye Teh for reviewing the paper prior to submission.

\bibliography{main}
\bibliographystyle{icml2024}

\newpage
\onecolumn
\appendix
\section{Category Theory Basics}



The natural notion of `sameness` for categories is \textit{equivalence}:

\begin{definition}[Equivalence]
    An \newdef{equivalence} between two categories $\cC$ and $\cD$, written $\cC \overset{\sim}{\rightarrow} \cD$, consists of a pair of functors $F: \cC \rightarrow \cD$ and $G: \cD \rightarrow \cC$ together with natural isomorphisms (natural transformations where every component has an inverse) $F \circ G \cong 1_{\cD}$ and $G \circ F \cong 1_{\cC}$.
\end{definition}

\begin{definition}[Initial object]
    \label{def:initial_object}
    An object $I$ in a category $\cC$ is called \newdef{initial} if for every $X \in \cC$ it naturally holds that $\cC(I, X) \cong 1$, meaning that there is only one map of type $I \to X$.
\end{definition}



\begin{definition}[Terminal object]
    \label{def:terminal_object}
    An object $T$ in a category $\cC$ is called \newdef{terminal} if for every $X \in \cC$ it naturally holds that $\cC(X, T) \cong 1$, meaning that there is only one map of type $X \to T$.
\end{definition}

\begin{definition}[Limit]\label{def:limit}
    For categories $\cJ$ and $\cC$, a \textbf{diagram of shape $\cJ$} in $\cC$ is a functor $D : \cJ \to \cC$. A \textbf{cone} to a diagram $D$ consists of an object $C \in \cC$ and a natural transformation from a functor constant at $C$ to the functor $D$, i.e.\ a family of morphisms $c_j: C \to D(j)$ for each object $j \in \cJ$, such that for any $f: i \to j$ in $\cC$ the following diagram commutes:
  \[\begin{tikzcd}[ampersand replacement=\&]
    C \& D(i) \\
    \& D(j)
    \arrow["D(f)", from=1-2, to=2-2]
    \arrow["c_i", from=1-1, to=1-2]
    \arrow["c_j"', from=1-1, to=2-2]
  \end{tikzcd}\]
  A morphism of cones $\theta: (C, c_j) \to (C', c_j')$ is a morphism $\theta$ in $\cC$ making each diagram 
    \[\begin{tikzcd}[ampersand replacement=\&]
    C \& C' \\
    \& D(j)
    \arrow["c'_j", from=1-2, to=2-2]
    \arrow["\theta", from=1-1, to=1-2]
    \arrow["c_j"', from=1-1, to=2-2]
  \end{tikzcd}\]
  commute. The \textbf{limit} of a diagram $D$, written $\varprojlim D$, is the terminal object in the category $\mathsf{Cone}(D)$ of cones to $D$ and morphisms between them.
    
\end{definition}

\begin{definition}[Colimit]\label{def:colimit}
  The \textbf{colimit} $\colim D$ of a diagram $D: \cJ^{op} \to \cC$ is the initial object in the category $\mathsf{Cocone}(D)$ of cocones to $D$, where a cocone $(C, c_j: D(j) \to C)$ has the dual property of a cone (above) with the morphisms reversed. \textbf{Small} (respectively $\kappa$-directed, connected, ...) colimits are colimits for which the indexing category $\cJ$ is small (respectively $\kappa$-directed, connected, ...).
\end{definition}

\begin{example}
    A terminal object in $\cC$ is a limit of the unique diagram from the empty category to $\cC$. Similarly an initial object is an example of a colimit.
\end{example}

\section{1-Categorical Algebra}

\begin{definition}[Monad coherence diagrams]
    \label{def:monad_diagrams}
    A triple $(M,\eta,\mu)$ of:
    \begin{itemize}
    \item an endofunctor $M : \cC \to \cC$;
    \item a natural transformations $\eta : \id_\cC \Rightarrow M$; and
    \item a natural transformation $\mu : M \circ M \Rightarrow M$
    \end{itemize}
    constitute a \newdef{monad} if the following diagrams commute:
    
\[\begin{tikzcd}[ampersand replacement=\&]
	M \& {M\circ M} \& {M \circ M \circ M} \& {M \circ M} \\
	\& M \& {M\circ M} \& M
	\arrow[Rightarrow, no head, from=1-1, to=2-2]
	\arrow["\mu", from=1-2, to=2-2]
	\arrow["\eta", from=1-1, to=1-2]
	\arrow["\mu"', from=1-3, to=2-3]
	\arrow["\mu", from=1-4, to=2-4]
	\arrow["{M\circ \mu}", from=1-3, to=1-4]
	\arrow["\mu"', from=2-3, to=2-4]
\end{tikzcd}\]

\end{definition}

\begin{definition}[Coalgebra for an endofunctor]
    \label{def:coalgebra_for_endofunctor}
    Let $\cC$ be a category and $F$ an endofunctor on $\cC$.
    A coalgebra for $F$ is a pair $(A, a)$ where $A$ is an object of $\cC$ and $a : A \to F(A)$ is a morphism of $\cC$.
\end{definition}

\subsection{Well-pointed Endofunctors and Algebraically Free Monads}

In this subsection we'll relate endofunctors and their algebras to monads and their algebras by way of \emph{well-pointed endofunctors} and the transfinite construction for an \emph{algebraically free monad} on such an endofunctor.

\begin{definition}
  A \newdef{pointed endofunctor} $(F,\sigma)$ comprises an endofunctor $F : \cC \to \cC$, and a natural transformation $\sigma : \id_{\cC} \Rightarrow F$.
  A pointed endofunctor is said to be \newdef{well-pointed} if the whiskering of $F$ and $\sigma$ either from the left or the right gives the same result, or in other words, if the natural transformations $\sigma \circ F$ and $F \circ \sigma$ (which are of type $F \Rightarrow F \circ F$) are equal.

  An \newdef{algebra for a pointed endofunctor} $(F, \sigma)$ is an algebra $(A,a)$ for the endofunctor $F$ for which the following diagram commutes:
  \[\begin{tikzcd}[ampersand replacement=\&]
    A \& FA \\
    \& A
    \arrow["a", from=1-2, to=2-2]
    \arrow["{\sigma_A}", from=1-1, to=1-2]
    \arrow["{\id_A}"', from=1-1, to=2-2]
  \end{tikzcd}\]

\end{definition}

\begin{remark}
  As well-pointedness is a property, and not structure, algebras for well-pointed endofunctors are just algebras for the pointed endofunctors.
  Morphisms of algebras for pointed endofunctors are morphisms of algebras for the underlying endofunctor.
\end{remark}

\begin{example}[Monads and pointed endofunctors, idempotent monads and well-pointed endofunctors]
  The unit and underlying endofunctor of a monad constitute a pointed endofunctor. If moreover that monad is idempotent, then that pointed endofunctor is well-pointed.
\end{example}

\begin{definition}
  For a given endofunctor $F$, its algebras and algebra homomorphisms form a category we denote by $\AlgEndo{F}$. For $(F,\sigma)$ a (well)-pointed endofunctor we denote by $\AlgPEndo{F}$ the category of algebras for $F$ and homomorphisms thereof. For $(F,\mu,\eta)$ a monad its algebras and homomorphisms thereof form a category we denote by $\AlgMnd{F}$.
\end{definition}

\begin{lemma}
  \label{lemma:endo_alg_to_pointed_endo_alg_iso}
  Suppose $F$ is an endofunctor on a category $\cC$ with coproducts. Then there is an equivalence of categories $\AlgEndo{F} \overset{\sim}{\rightarrow} \AlgPEndo{F+\id_{\cC}}$.
\end{lemma}

\begin{definition}
  \label{def:alg_free_monad}
  Given an endofunctor $F:\cC \rightarrow \cC$, an \newdef{algebraically free monad} on $F$ is a monad ${\Free_{\Mnd}(F)}$ together with an equivalence of categories $\AlgEndo{F} \overset{\sim}{\rightarrow}  \AlgMnd{\Free_{\Mnd}(F)}$ which preserves the respective functors to $\cC$ that forget the algebraic structure.
\end{definition}

\subsection{Kelly's Unified Transfinite Construction}\label{sec:Kellys_transfinite_construction}

The existence theorem for algebraically free monads is Kelly's unified transfinite construction \citep{Kelly_1980}.

\begin{definition}[Reflective subcategory]\label{def:reflective_subcategory}
    A full subcategory $\cD$ of a category $\cC$ is \textbf{reflective} if the inclusion functor $F : \cD \to \cC$ admits a left adjoint $G : \cC \to \cD$. A reflective subcategory of a presheaf category is called a \textbf{locally presentable category}.
\end{definition}
\begin{example}[Categories are reflective in graphs]
    The category of small categories is a reflective subcategory of the category of graphs, which is itself a presheaf category.
    \end{example}
\begin{example}[Ubiquity of local presentability]
    Nearly every category often encountered in practice is a locally presentable category. 
    The categories of monoids, groups, rings, vector spaces, and modules are locally presentable. As are the categories of topological spaces, manifolds, metric spaces, and uniform spaces. 
    While its beyond the scope of this document to expound too much upon it, local presentability is a particularly powerful notion of what it means for objects and morphisms to be of things defined by equalities of set-sized expressions.
\end{example}


\begin{definition}[Accessible category and accessible functor]
    For an ordinal $\kappa$, a $\kappa$-\textbf{accessible category} $\cC$ is a category such that:
    \begin{itemize}
        \item $\cC$ has $\kappa$-directed colimits; and
        \item there is a set of $\kappa$-compact objects which generates $\cC$ under $\kappa$-directed colimits.
    \end{itemize}
    A $\kappa$-\textbf{accessible functor} $F:\cC \rightarrow \cD$ is an functor between $\kappa$-accessible functor which preserves $\kappa$-filtered colimits.  
\end{definition}

\begin{remark}
    The accessibility of a functor can be thought of as an upper-bound on the arity of the operations which it abstracts. For example finite sums of finite sums are again finite sums.
\end{remark}

\begin{definition}
    Let $\cC$ be a $\kappa$-accessible locally presentable category and $F:\cC \rightarrow \cC$ a pointed $\kappa$-accessible endofunctor. Let $F^\kappa$ be the $\kappa$-directed colimit of the diagram \[F^0 \rightarrow F^1 \rightarrow F^2 \rightarrow \dots \rightarrow F^\kappa\] where $F^0 = \id_\cC$ and $F^{\alpha+1} = F \circ F^\alpha$ for $\alpha < \kappa$ and $F^\alpha = \colim_{\beta < \alpha} F^\beta$ for $\alpha$ a limit ordinal.
\end{definition}

\begin{lemma}\label{lem:transfinite_monad_from_well_pointed_endo_construction}
    For $\cC$ and $F$ as above, $F^\kappa$ is a monad. The unit is the canonical inclusion of $\id_\cC$ into the colimit $F^\kappa$ and the multiplication comes from the preservation of $\kappa$-filtered colimits by $F$.
\end{lemma}


\begin{theorem}
    \label{Kelly_unified_applied_to_well_pointed_endo} Assume the hypotheses of \cref{lem:transfinite_monad_from_well_pointed_endo_construction} with $\kappa$ the ordinal in those hypotheses. Then $\AlgPEndo{F}$ is equivalent to $\AlgMnd{F^{\kappa}}$ - i.e. $F^{\kappa}$ is an algebraically free monad for $F$.
\end{theorem}

\begin{remark}
  For the endofunctors $F$ we study here, $\Free(F) \overset{\sim}{\rightarrow} (F+\id)^\omega$ is the underlying endofunctor of an algebraically free monad for $F$, $\Free(F)$.
\end{remark}

The above formula can be related to the explicit formula for computing free monads, which has a dual formula in the case of cofree comonads \citep{ghani_algebras_2001}.

\begin{proposition}[(Co)free (co)monads, explicitly]
  \label{prop:free_monads_explicitly}
  Let $F : \cC \to \cC$ be an endofunctor.
  Then $\Free_{\Mnd}(F)$, the free monad on $F$ is given by $\Free_{\Mnd}(F)(Z) = \Fix(X \mapsto F(X) + Z)$.
  Dually, we compute $\Cofree_{\Cmnd}(F)$, the cofree comonad on $F$ as $\Cofree_{\Cmnd}(F)(Z) = \Fix(X \mapsto F(X) \times Z)$.
\end{proposition}

\begin{example}[Free monad on $1 + A \times -$]
  \label{ex:free_monad_on_lists}
  Free monad on the endofunctor $1 + A \times -$ is $\ListM{-}{A} : \Set \to \Set$, the endofunctor mapping an object $Z$ to the set $\ListM{Z}{A}$ of lists of elements of type $A$ whose last element is not $[]$, i.e.\ an element of type $1$, but instead an element of type $Z + 1$.
  That is, these are lists which end potentially with an element of $Z$.
\end{example}

\begin{example}[Free monad on $A + (-)^2$]
  \label{ex:free_monad_on_trees}
  The free monad on the endofunctor $A + (-)^2$ is given by $\Tree{A + -} : \Set \to \Set$, mapping a set $Z$ to the set of trees with $A + Z$ labelled leaves.
\end{example}

\begin{example}[Cofree comonad on $O \times -$]
  \label{ex:cofree_comonad_stream}
  The cofree comonad on the endofunctor $O \times -$ is $\Stream{O \times -}$, mapping an object $Z$ to the set of streams whose outputs are of type $O \times Z$.
\end{example}

\begin{example}[Cofree comonad on $(I \to O \times -)$]
  \label{ex:cofree_comonad_on_mealy}
  The cofree comonad of the endofunctor $(I \to O \times -)$ is $\Fix(X \mapsto (I \to O \times X) \times -)$, mapping a set $Z$ to a set of hybrids of Moore and Mealy machines, outputting an additional element of $Z$ at each step which does not depend on $I$.
\end{example}

\section{Additional Geometric Deep Learning Examples}

\label{app:gdl_further}

To further illustrate the power of categorical deep learning as a framework that subsumes geometric deep learning \citep{bronstein2021geometric}, as well as make the reader more comfortable in manipulating monad algebras and their homomorphisms, we provide three additional examples deriving equivariance constraints of established geometric deep learning architectures, leveraging the framework of CDL.

All of these examples should be familiar to Geometric DL practitioners, and are covered in detail by prior papers \citep{maron2018invariant,cohen2018spherical,thomas2018tensor,cohen2016group}, hence we believe that relegating their exact derivations to appendices is appropriate in our work.

Before we begin, we recall the core template of our work: that we represent neural networks $f : A\rightarrow B$ as monad algebra homomorphisms between two algebras $(A,a)$ and $(B,b)$, for a monad $(M, \eta, \mu)$:
\[\begin{tikzcd}[ampersand replacement=\&]
  {M(A)} \& {M(B)} \\
  A \& B
  \arrow["a"', from=1-1, to=2-1]
  \arrow["b", from=1-2, to=2-2]
  \arrow["f"', from=2-1, to=2-2]
  \arrow["{M(f)}", from=1-1, to=1-2]
\end{tikzcd}\]
and that geometric deep learning can be recovered by making our monad be the \emph{group action monad}; $M(X) =G\times X$.

\subsection{\emph{Permutation}-equivariant Learning on \emph{Graphs}}
leading to graph neural networks \citep{velivckovic2023everything}.
\[\begin{tikzcd}
	{\Sigma_n\times\mathbb{R}^n\times\mathbb{R}^{n\times n}} && {\Sigma_n\times\mathbb{R}^n} \\
	\\
	{\mathbb{R}^n\times\mathbb{R}^{n\times n}} && {\mathbb{R}^n}
	\arrow["{\Sigma_n\times f}", from=1-1, to=1-3]
	\arrow["{P_{X,A}}"', from=1-1, to=3-1]
	\arrow["{P_X}", from=1-3, to=3-3]
	\arrow["f"', from=3-1, to=3-3]
\end{tikzcd}\]
In this case:
\begin{itemize}
    \item The group $G=\Sigma_n$ is the permutation group of $n$ elements,
    \item The carrier object for the algebras includes (scalar) node features $\mathbb{R}^n$ and, potentially, adjacency matrices $\mathbb{R}^{n\times n}$,
    \item The structure map for the first algebra, $P_{X,A} : \Sigma_n\times\mathbb{R}^n\times\mathbb{R}^{n\times n}\rightarrow \mathbb{R}^n\times\mathbb{R}^{n\times n}$, executes the permutation: $P_{X,A}(\sigma, {\bf X}, {\bf A}) = ({\bf P}(\sigma){\bf X}, {\bf P}(\sigma){\bf A}{\bf P}(\sigma)^\top)$, where ${\bf P}(\sigma)$ is the permutation matrix specified by $\sigma$.
    \item The structure map for the second algebra, $P_X$, still executes the permutation, but only over the node features. This reduction is not strictly necessary, but is often the standard when designing graph neural networks—as they are often assumed to not modify their underlying computational graph. We deliberately assume this reduction here to illustrate how our framework can handle neural networks transitioning across different algebras.
\end{itemize}

\subsection{\emph{Rotation}-equivariant Learning on \emph{Spheres}}
leading to the first layer of spherical CNNs \citep{cohen2018spherical}.
\[\begin{tikzcd}
	{\mathrm{SO}(3)\times (S^2\rightarrow\mathbb{R})} && {\mathrm{SO}(3)\times (\mathrm{SO(3)}\rightarrow\mathbb{R})} \\
	\\
	{S^2\rightarrow\mathbb{R}} && {\mathrm{SO}(3)\rightarrow\mathbb{R}}
	\arrow["{\mathrm{SO}(3)\times f}", from=1-1, to=1-3]
	\arrow["{\rho_{S^2}}"', from=1-1, to=3-1]
	\arrow["{\rho_{\mathrm{SO}(3)}}", from=1-3, to=3-3]
	\arrow["f"', from=3-1, to=3-3]
\end{tikzcd}\]
In this case:
\begin{itemize}
    \item The group $G = \mathrm{SO}(3)$ is the special orthogonal group of 3D rotations.
    \item The carrier object for the first algebra is $S^2\rightarrow\mathbb{R}$; (scalar) data defined over the sphere. Note that in practice, this will usually be discretised, so we will be able to represent it using matrices.
    \item The structure map of the first algebra, $\rho_{S^2} : \mathrm{SO}(3)\times (S^2\rightarrow\mathbb{R})\rightarrow (S^2\rightarrow\mathbb{R})$, executes a 3D rotation on the spherical data, as follows: $\rho_{S^2}((\alpha, \beta, \gamma), \psi) = \phi$, such that $\phi({\bf x}) = \psi({\bf R}({\alpha, \beta, \gamma})^{-1}{\bf x})$, where ${\bf R}({\alpha, \beta, \gamma})$ is the rotation matrix specified by the ZYZ-Euler angles $(\alpha, \beta, \gamma)\in\mathrm{SO}(3)$. This (inverse) rotation is applied to points ${\bf x}\in S^2$ on the sphere.
    \item The carrier object for the second algebra is $\mathrm{SO}(3)\rightarrow\mathbb{R}$; (scalar) data defined over rotation matrices. Once again, this will usually be discretised in practice.
    \item The structure map of the second algebra, $\rho_{\mathrm{SO}(3)} : \mathrm{SO}(3)\times (\mathrm{SO}(3)\rightarrow\mathbb{R})\rightarrow (\mathrm{SO}(3)\rightarrow\mathbb{R})$ now executes a 3D rotation over the rotation-matrix data, as follows: $\rho_{\mathrm{SO}(3)}((\alpha, \beta, \gamma), \psi) = \phi$ such that $\phi({\bf R}({\alpha', \beta', \gamma'})) = \psi({\bf R}({\alpha, \beta, \gamma})^{-1}{\bf R}({\alpha', \beta', \gamma'}))$.
\end{itemize}

\subsection{$G$-equivariant Learning on $G$}
leading to G-CNNs \citep{cohen2016group}, as well as the subsequent layers of spherical CNNs \citep{cohen2018spherical}.
\[\begin{tikzcd}
	{G\times (G\rightarrow\mathbb{R})} && {G\times (G\rightarrow\mathbb{R})} \\
	\\
	{G\rightarrow\mathbb{R}} && {G\rightarrow\mathbb{R}}
	\arrow["{G\times f}", from=1-1, to=1-3]
	\arrow["{A_G}"', from=1-1, to=3-1]
	\arrow["{A_G}", from=1-3, to=3-3]
	\arrow["f"', from=3-1, to=3-3]
\end{tikzcd}\]
In this case:
\begin{itemize}
    \item The group $G$ is also the domain of the carrier objects ($G\rightarrow\mathbb{R}$).
    \item Both algebras' structure map follows the execution of the regular representation of $G$, $A_G : G\times (G\rightarrow\mathbb{R})\rightarrow(G\rightarrow\mathbb{R})$, by composition, as follows: $A_G(g, \psi)(h) = \psi(g^{-1}h)$.
\end{itemize}
\section{Lawvere Theories and Syntax}

\label{app:lawvere_theories}

\citet{de2020natural} expanded the theory of equivariant layers in neural networks using the abstraction of natural transformations of functors. In this section, we will explain how to understand morphisms of monad algebras in the same terms.

Indeed, this comparison is crucial to understanding the \emph{syntax} of monads, in addition to the \emph{semantics} given by their category of algebras.

\begin{definition}
If $\mathcal{C}$ is a category, a \newdef{presheaf on $\mathcal{C}$} is a functor $\mathcal{C}^{op} \to \Set$. The category whose objects are presheaves and morphisms are natural transformations is denoted $\Psh(\mathcal{C})$.
\end{definition}

Fix a monad $T$ on $\Set$ and let $\mathcal{C}_T$ denote its category of algebras and algebra homomorphisms.

Given an algebra $A\in\mathcal{C}_T$, the easiest way to interpret $A$ as a functor is via the \emph{Yoneda embedding} $\mathcal{C}_T\to \Psh(\mathcal{C}_T)$, which identifies $A$ with the presheaf $[-,A]$. It is a standard result that the Yoneda functor is fully faithful, which means that we can identify morphisms as algebras with morphisms as presheaves.

Furthermore, these presheaves have a special property. If $\varinjlim j$ is a small colimit in $\mathcal{C}_T$, then $[\varinjlim j, A]=\varprojlim [j, A]$, essentially by the definition of limits and colimits.

\begin{definition}
    If $\mathcal{C}$ is a category and $J$ is a class of small colimits in $\mathcal{C}$, then a \newdef{$J$-continuous presheaf on $\mathcal{C}$} is a presheaf $F\in\Psh(\mathcal{C})$ satisfying $F(\varinjlim j)=\varprojlim F(j)$ for all $j\in J$. The corresponding full subcategory of $\Psh(\mathcal{C})$ is denoted $\Sh_J(\mathcal{C})$, or simply $\Sh(\mathcal{C})$ in the case that $J$ is the class of all small colimits. We will also use $\coprod$ to refer to the class of all coproducts, $+$ to refer to the class of binary coproducts, and $\emptyset$ to refer to the empty class.
\end{definition}

It turns out that $\mathcal{C}_T$ has a nice property: it is a \emph{strongly compact} category, meaning that every continuous presheaf is representable as above by an object of $\mathcal{C}_T$. In other words, we have the identification $\mathcal{C}_T = \Sh(\mathcal{C}_T)$.

However, it is usually impractical to work with the entire category $\mathcal{C}_T$. When possible, we want to reason in terms of a more tractable subcategory. These will provide us with workable \emph{syntax} for our monad. Following \cite{brandenburg2021}, we make use of Ehresmann's concept of a ``colimit sketch'': a category equipped with a restricted class of colimits. Rather than giving the general definition, we will describe a few special cases of prime interest.

\begin{definition}
    The \newdef{Kleisli category $\mathcal{K}_T$} is the full subcategory of $\mathcal{C}_T$ on the free algebras. The \newdef{finitary Kleisli category $\mathcal{K}^\mathbb{N}_T$} is the full subcategory on the free algebras of the form $TS$ for $S$ a finite set. And the \newdef{unary Kleisli category $\mathcal{K}^1_T$} is the full one-object subcategory on the free algebra $T1$.
\end{definition}

We have the following sequence of nested full subcategories:

\[\mathcal{C}_T \supset \mathcal{K}_T \supset \mathcal{K}^\mathbb{N}_T \supset \mathcal{K}^1_T\]

By composing with these inclusions, we get a sequence of functors between presheaf categories:

\[\Psh(\mathcal{C}_T) \to \Psh(\mathcal{K}_T) \to \Psh(\mathcal{K}^\mathbb{N}_T) \to \Psh(\mathcal{K}^1_T)\]

By restricting our class of colimits, we can restrict this to the corresponding continuous presheaf categories:

\[\mathcal{C}_T \cong \Sh(\mathcal{C}_T) \to \Sh_{\amalg}(\mathcal{K}_T) \to \Sh_{+}(\mathcal{K}^\mathbb{N}_T) \to \Sh_{\emptyset}(\mathcal{K}^1_T)\]

It turns out that in many cases, these arrows are equivalences of categories.

\begin{definition}
    In the case that $\mathcal{C}_T \to \Sh_{\amalg}(\mathcal{K}_T)$ is an equivalence, we say that $\mathcal{K}_T^{op}$ is the \newdef{infinitary Lawvere theory} for $T$.
\end{definition}

\begin{remark}
In fact, all monads have an infinitary Lawvere theory. For a proof, see \cite{brandenburg2021}.
\end{remark}

\begin{definition}
    In the case that $\mathcal{C}_T \to \Sh_{+}(\mathcal{K}_T^{\mathbb{N}})$ is an equivalence, we say that $(\mathcal{K}_T^{\mathbb{N}})^{op}$ is the \newdef{Lawvere theory} for $T$, and $T$ is a \newdef{finitary monad}.
\end{definition}

\begin{example}
If $T$ is the monad sending a set $S$ to the free commutative semiring on $S$, then $\mathcal{C}_T$ is the category of commutative semirings. Since the axioms for commutative semirings consist of equations with a finite number of variables, $T$ has a Lawvere theory $(\mathcal{K}_T^{\mathbb{N}})^{op}$.

It is well known that this is the category of ``finite polynomials'', see e.g. \cite{gambino2013polynomial}. This connection was observed in \cite{dudzik2022graph} to relate message passing in Graph Neural Networks to polynomial functors.

In fact, we could further restrict this theory to just the four objects $\{T0, T1, T2, T3\}$, because we can fully axiomatise commutative semirings with only three variables, e.g. $a(b+c)=ab+ac$.
\end{example}

\begin{example}[Monoids]
    If $\mathcal{C}_T \to \Sh_{\emptyset}(\mathcal{K}_T^1)$ is an equivalence, then in fact $\mathcal{C}_T \cong \Psh(\mathcal{K}_T^1)$, which is the category of $M$-sets for the monoid $M=(\mathcal{K}_T^1)^{op}$. So we can see that ``unary Lawvere theories'' exactly correspond to monads of the form $M\times -$, where $M$ is a monoid.
\end{example}

\begin{example}[Suplattices]
    Not all monads are finitary; that is, not all monads have an associated Lawvere theory.

    For an example of a non-finitary monad, let $\mathcal{P}:\Set\to\Set$ be the covariant powerset functor. That is, $\mathcal{P}(S)$ is the set of all subsets of $S$, and if $f:S\to T$ is a function, then $\mathcal{P}(f):\mathcal{P}(S)\to\mathcal{P}(T)$ is defined by $\mathcal{P}(f)(A) := \{f(a) \mid a\in A\}$.

    We equip $\mathcal{P}$ with the monad structure given by the unit $1\to\mathcal{P}$ sending $s\in S$ to $\{s\}\in\mathcal{P}(S)$ and the composition $\mathcal{P}^2\to\mathcal{P}$ sending $\mathcal{A}\subset \mathcal{P}(S)$ to $\bigcup_{A\in \mathcal{A}} A$.

    The category of $\mathcal{P}$-algebras $\mathcal{C}_\mathcal{P}$ is the category of \newdef{suplattices}, which is the category of posets with all least upper bounds, and morphisms preserving them. Equivalently, a suplattice is given by a set $L$ together with a join map $\bigvee \mathcal{P}(L)\to L$ satisfying unit and composition axioms.

    Implementing $\bigvee$ requires implementing operations $L^\kappa \to L$ for arbitrarily large cardinal numbers $\kappa$, so we can see intuitively that suplattices are not described by a finitary theory.

    Note that the Kleisli category $\mathcal{K}_\mathcal{P}$ is equivalent to the category of sets and relations.
\end{example}

\begin{example}[Semilattices]
    While $\mathcal{P}$ above isn't finitary, it has a finitary counterpart $\mathcal{P}_{fin}$, the functor that takes each set to its \emph{finite} subsets.
    
    Algebras for $\mathcal{P}_{fin}$ are sometimes called \newdef{join-semilattices}, partial orders where every finite subset has a least upper bound. It is a nice exercise to show that a join-semilattice is equivalently a commutative idempotent monoid.
\end{example}

\section{Monoidal Categories and Actegories}

Monoidal categories and actegories are the `categorified` version of monoids and actions.

\begin{definition}[{Strict monoidal category, \citep[Def. 1.2.1]{johnson_2-dimensional_2021}}]
    \label{def:strict_monoidal_category}
    Let $\cM$ be a category.
    We call $\cM$ a \newdef{strict monoidal category} if it is equipped with the following data a functor $\otimes : \cM \times \cM \to \cM$ called \emph{the monoidal product}; an object $I \in \cM$ called \newdef{the monoidal unit},
    such that
    \begin{itemize}
      \item $A \otimes (B \otimes C) = (A \otimes B) \otimes C$ for all $A, B, C \in \cM$;
      \item $A \otimes I = A = I \otimes A$ for all $A \in \cM$;
      \item $f \otimes (g \otimes h) = (f \otimes g) \otimes h$ for all $f, g, h \in \cM$;
      \item $\id \otimes f = f = f \otimes \id$ for all $f \in \cM$.
    \end{itemize}
\end{definition}

\begin{definition}[{Actegories, see \cite{capucci_actegories_2023}}]\label{def:actegory}
  Let $(\cM, \otimes, I, \alpha , \lambda, \rho)$ be a monoidal category.
  An $\cM$-\newdef{actegory} $\cC$ is a category $\cC$ together with a functor $\act : \cM \times \cC \rightarrow \cC$ together with natural isomorphisms $\eta_X : I \act X \cong X $ and $\mu_{M, N} : (M \otimes N) \act X \cong M \act (N \act X)$, such that:
    \begin{itemize}
      \item \textbf{Pentagonator.} For all $M, N, P\in\cM$ and $C\in\cC$ the following diagram commutes.

  \begin{equation}
    \label{eq:act_pentagonator}
    \begin{tikzcd}[ampersand replacement=\&]
	{ (M \otimes N ) \otimes P \act C } \&\& {M \otimes N \act P \act C} \& {} \\
	{} \&\&\&\& {M \act N \act P \act C} \\
	{M \otimes (N \otimes P) \act C  } \&\& {M \act N \otimes P \act C} \& {}
	\arrow["{\mu_{M \otimes N, P, C}}", from=1-1, to=1-3]
	\arrow["{\alpha_{M, N, P} \act C}"', from=1-1, to=3-1]
	\arrow["{\mu_{M,N,P \act C}}", from=1-3, to=2-5]
	\arrow["{\mu_{C, M \otimes N, P}}"', from=3-1, to=3-3]
	\arrow["{M \act \mu_{N,P,C} }"', from=3-3, to=2-5]
\end{tikzcd}
  \end{equation}
    \item \textbf{Left and right unitors.} For all $C\in\cC$ and $M\in\cM$
      The diagrams below commute.
  \begin{equation}
    \begin{tikzcd}[ampersand replacement=\&]
	{I \otimes M \act C} \&\&\&\&\& {(M \otimes I) \act C} \\
	\\
	{I\act M \act C} \&\& {M \act C} \& {M \act C} \&\& {M \act I \act C}
	\arrow["{\mu_{I,M,C}}"', from=1-1, to=3-1]
	\arrow["{\lambda_M \act C}", from=1-1, to=3-3]
	\arrow["{\rho_M \act C}"', from=1-6, to=3-4]
	\arrow["{\mu_{M,I,C}}", from=1-6, to=3-6]
	\arrow["{\eta_{ M \act C}}"', from=3-1, to=3-3]
	\arrow["{M \act \eta_C}", from=3-6, to=3-4]
\end{tikzcd}
  \end{equation}
      
  \end{itemize}
    
  \end{definition}

\begin{remark}
   Just as one may assume a monoidal category to be strict, via MacLane's coherence theorem, we may assume actegories are strict as well \citep[Remark 3.4]{capucci_actegories_2023}. That is to say, we may assume for an $\cM$ actegory $(\cC,\act)$ that:
  \begin{itemize}
    \item The unitor $\eta_X$ is an equality, i.e.\ that $I\act (\_) = \id_{\cC}$ are equal as functors of type $\cC \rightarrow \cC$; and
    \item The multiplicator $\mu_{M, N}$ is an equality, i.e.\ that $((\_) \otimes (\_) \act (\_) = (\_) \act (\_) \act (\_)$ are equal as functors of type $ \cM \times \cM \times \cC \rightarrow \cC$
  \end{itemize}
  We will call such a structure a \newdef{strict actegory}. We will use these later on to simplify exposition and some theorems.
\end{remark}

\begin{example}[Monoidal action]
Any monoidal category gives rise to a self-action.
\end{example}

\begin{example}[Families actegories]
  Any category $\cC$ with coproducts has an action $\act : \Set \times \cC \to \cC$ which maps $(X, A)$ to the coproduct of $\left| X \right|$ copies of $A$.
\end{example}

\subsection{Morphisms of Actegories}

\begin{definition}[Actegorical strong monad]
  Let $(\cC, \act)$ be a $\cM$-actegory.
  A monad $(T,\mu,\eta)$ on  $\cC$ is called \newdef{strong}\footnote{Another name for this is a $\cM$-linear morphism, used in \cite{capucci_actegories_2023}} if it is equipped with a natural transformation $\sigma_{P,A}:P \act T(A) \rightarrow T(P \act A)$, called \textbf{strength} such that diagrams in \cref{def:actegorical_monad_strength} commute.
\end{definition}


\begin{example}
  All monads of the form $A \times - : \Set \to \Set$ are strong for the actegory $(\Set,\times)$, for any monoid $A$.
  This includes $G \times - : \Set \to \Set$ from \cref{ex:group_action_monad}.\footnote{The strength is in fact a natural isomorphism.}
\end{example}

\begin{definition}[Actegorical strong endofunctor]
  \label{def:actegorical_endofunctor}
  Let $(\cC, \act)$ be a $\cM$-actegory.
  An endofunctor $F$ on $\cC$ is called \newdef{strong} if it is equipped with a natural transformation $\sigma_{P,A}:P \act F(A) \rightarrow F(P \act A)$, called a \textbf{strength} such that diagrams \textbf{AS1} and \textbf{AS2} in \cref{def:actegorical_monad_strength} commute.
\end{definition}

\begin{definition}[Actegorical strong monad coherence diagrams]
    \label{def:actegorical_monad_strength}
    A monad $(T,\mu,\eta)$ on the category $\cC$ of a $\cM$-actegory $(\cC, \act)$ is called \newdef{strong} if it is equipped with a natural transformation $\sigma_{P,A}:P \act T(A) \rightarrow T(P \act A)$, called \textbf{strength} making the diagrams below commute:
    
    \begin{minipage}[t]{.5\textwidth}
        \textbf{AS1: Compatibility of strength and the monoidal unit.}
                \[\begin{tikzcd}[ampersand replacement=\&]
                    {I \act T(A)} \& {T(I \act A )} \\
                    \& {T(A)}
                    \arrow["{\sigma_{I,A}}", from=1-1, to=1-2]
                    \arrow["{T(\lambda_A)}", from=1-2, to=2-2]
                    \arrow["{\lambda_{T(A)}}"', from=1-1, to=2-2]
                \end{tikzcd}\]
    \end{minipage}%
    \begin{minipage}[t]{0.5\textwidth}
        \textbf{AS2: Compatibility of strength and actegory multiplicator.}
        \[\begin{tikzcd}[ampersand replacement=\&]
        	{(P \otimes Q) \act T(X) } \&\& {P \act (Q \act T(X))} \\
        	\&\& {P \act T(Q \act X)} \\
        	{T((P \otimes Q) \act X)} \&\& {T(P \act (Q \act X))}
        	\arrow["{\sigma_{P\otimes Q,X}}"', from=1-1, to=3-1]
        	\arrow["{\alpha_{P,Q,T(X)}}", from=1-1, to=1-3]
        	\arrow["{P \act \sigma_{Q,X}}", from=1-3, to=2-3]
        	\arrow["{\sigma_{P,Q \act X}}", from=2-3, to=3-3]
        	\arrow["{T(\alpha_{P,Q,X})}"', from=3-1, to=3-3]
        \end{tikzcd}\]
    \end{minipage}
    \begin{minipage}[t]{.5\textwidth}
        \textbf{AS3: Compat. of strength and monad multiplication.}
       \[\begin{tikzcd}[ampersand replacement=\&]
       	{A \act T(T(X))} \&\& {A \act T(X)} \\
       	{T(A \act T(X))} \\
       	{T(T(A \act X))} \&\& {T(A \act X)}
       	\arrow["{A \act \mu_X}", from=1-1, to=1-3]
       	\arrow["{\sigma_{A, X}}", from=1-3, to=3-3]
       	\arrow["{\sigma_{A, T(X)}}"', from=1-1, to=2-1]
       	\arrow["{T(\sigma_{A, T(X)})}"', from=2-1, to=3-1]
       	\arrow["{\mu_{A \act X}}"', from=3-1, to=3-3]
       \end{tikzcd}\] 
    \end{minipage}%
    \begin{minipage}[t]{0.5\textwidth}
        \textbf{AS4: Compatibility of strength and monad unit.}
        \[\begin{tikzcd}[ampersand replacement=\&]
        	\& {M \act X} \\
        	{M \act T(X)} \&\& {T(M \act X)}
        	\arrow["{M \act \eta_X}"', from=1-2, to=2-1]
        	\arrow["{\sigma_{M, X}}"', from=2-1, to=2-3]
        	\arrow["{\eta_{M \act X}}", from=1-2, to=2-3]
        \end{tikzcd}\]
    \end{minipage}
\end{definition}

\begin{example}
  \label{ex:strong_endofunctors}
  The examples of endofunctors from \cref{ex:list_as_algebra,ex:tree_as_algebra,ex:stream_as_coalgebra,ex:mealy_machine_as_coalgebra,ex:moore_machine_as_coalgebra} are all strong endofunctors, and their appropriate free monads are strong too.
  See Appendix (\cref{ex:strengths_endofunctors_free_monads})
\end{example}

\begin{example}
  \label{ex:strengths_endofunctors_free_monads}
  Below we present a list the data of \emph{strengths} $\sigma_{P, X}$ of some relevant endofunctors and (on the left) and their corresponding free monads (on the right).

\begin{minipage}[t]{.5\textwidth}
  \begin{center}
    Given the endofunctor $1 + A \times - : \Set \to \Set$, its strength $\sigma_{P, X} : P \times (1 + A \times X) \to 1 + A \times P \times X$ is given by:
  \end{center}
    \begin{align*}
        &\sigma_{P, X}(p, \inl(\bullet)) = \inl(\bullet)\\
        &\sigma_{P, X}(p, \inr(a, x)) = \inr(a, p, x)
    \end{align*}
\end{minipage}%
\begin{minipage}[t]{0.5\textwidth}
  \begin{center}
    Given the free monad $\ListM{-}{A} : \Set \to \Set$, its strength $\sigma_{P, Z} : P \times \ListM{Z}{A} \to \ListM{P \times Z}{A}$ is given by:
  \end{center}
      \begin{align*}
        &\sigma_{P, Z}(p, \Nil) = \Nil\\
        &\sigma_{P, Z}(p, z) = (p, z)\\
        &\sigma_{P, Z}(p, \Cons(a, as)) = \Cons(a, \sigma_{P, Z}(p, as))
      \end{align*}
\end{minipage}

\begin{minipage}[t]{.5\textwidth}
  \begin{center}
    Given the endofunctor $A + (-)^2$, its strength $\sigma_{P, X} : P \times (A + X^2) \to A + (P \times X)^2$ is given by:
  \end{center}
    \begin{align*}
      &\sigma_{P, X}(p, \inl(a)) = \inl(a)\\
      &\sigma_{P, X}(p, \inr(x, x')) = \inr((p, x), (p, x'))\\
    \end{align*}
\end{minipage}%
\begin{minipage}[t]{0.5\textwidth}
  \begin{center}
    Given the free monad $\Tree{A + -} : \Set \to \Set$, its strength $\sigma_{P, Z} : P \times \Tree{A + Z} \to \Tree{A + P \times Z}$ is given by:
  \end{center}
      \begin{align*}
        &\sigma_{P, Z}(p, \Leaf(\inl(a))) = \Leaf(a)\\
        &\sigma_{P, Z}(p, \Leaf(\inr(z))) = \Leaf(\inr(p, z))\\
        &\sigma_{P, Z}(p, \Node(l, r)) = \Node(\sigma_{P, Z}(p, l),\sigma_{P, Z}(p, r))\\
      \end{align*}
\end{minipage}

\end{example}

\section{2-Categorical Algebra}

Good references for 2-categorical algebra are \cite{Lack_2_categories_companion} or \cite{Kelly_enriched_categories}. The latter deals with the more general notion of enriched categories of which 2-categories are a particular example.

\subsection{2-monads and their Lax Algebras}

A 2-monad can conscisely be defined as a $\Cat$-enriched monad (see \citet[Sec. 6.5]{johnson_2-dimensional_2021})
An unpacking of its definition follows.

\begin{definition}[2-monad]
  A \newdef{2-monad} on a 2-category $\cC$ comprises:
  \begin{itemize}
    \item A 2-endofunctor $T$ on $\cC$;
    \item A 2-natural transformation $\mu : T^2 \Rightarrow T$; and
    \item A 2-natural transformation $\eta : \id_{\cC} \Rightarrow T$;
  \end{itemize}
  such that the ($\Cat$-enriched variant of) the axioms in \cref{def:monad_diagrams} hold.

\end{definition}

\begin{definition}[Lax $T$-algebra for a 2-monad]
  Let $(T, \eta, \mu)$ be a 2-monad on $\cC$.
  A lax $T$-algebra is a pair $(A, r, \epsilon_A, \delta_A)$ where $A$ is an object of $\cC$, $a : T(A) \to A$ is a morphism in $\cC$, and $\epsilon_A$ and $\delta_A$ are the following 2-morphisms in $\cC$:
\[\begin{tikzcd}[ampersand replacement=\&]
	{T(A)} \&\& {T(T(A))} \& {T(A)} \\
	A \& A \& {T(A)} \& A
	\arrow["a"', from=1-1, to=2-1]
	\arrow[Rightarrow, no head, from=2-1, to=2-2]
	\arrow[""{name=0, anchor=center, inner sep=0}, "{\eta_A}"', from=2-2, to=1-1]
	\arrow["a", from=1-4, to=2-4]
	\arrow["a"', from=2-3, to=2-4]
	\arrow["{\mu_A}", from=1-3, to=1-4]
	\arrow["{T(a)}"', from=1-3, to=2-3]
	\arrow["{\delta_A}", Rightarrow, from=2-3, to=1-4]
	\arrow["{\epsilon_A}", shorten >=2pt, Rightarrow, from=2-1, to=0]
\end{tikzcd}\]
such that the lax unity and lax associativity conditions \citep[Eq. 6.5.6 and 6.5.7]{johnson_2-dimensional_2021} are satisfied.
\end{definition}

\begin{definition}[Lax algebra homomorphism]
  Let $\cE$ be a 2-category, $T$ a 2-monad on $\cE$, and $(A,a, \epsilon_A, \delta_A)$ and $(B,b, \epsilon_B, \delta_B)$ be lax $T$-algebras.
  Then a lax algebra morphism from $(A,a, \epsilon_A, \delta_A) \to (B,b, \epsilon_B, \delta_B)$ is pair $(f, \kappa)$ where  $f : A \rightarrow B$ is a morphism in $\cE$ and $\kappa$ is a 2-morphism spanned by the following diagram
\[\begin{tikzcd}[ampersand replacement=\&]
{T(A)} \& {T(B)} \\
A \& B
\arrow["a"', from=1-1, to=2-1]
\arrow["b", from=1-2, to=2-2]
\arrow["f"', from=2-1, to=2-2]
\arrow["{T(f)}", from=1-1, to=1-2]
\arrow["\kappa"', shorten <=2pt, shorten >=2pt, Rightarrow, from=1-2, to=2-1]
\end{tikzcd}\]

such that the diagrams in \citep[Def. 6.5.9]{johnson_2-dimensional_2021} commute.

Lax algebras and lax algebra homomorphisms for the 2-monad $T$ form the category $\AlgLaxMnd{T}$.
\end{definition}

\begin{example}[2-monad for $M$-modules]
  If $M$ is a monoid, it is useful to study ``$M$-modules'', which we define to be monoids equipped with a compatible action of $M$.\footnote{Contrary to a mere $M$-action $A$, where $A$ is a set, in a $M$-module $A$, $A$ comes equipped with a monoid structure too.} For example, the situation of a group acting on a vector space is fundamental in representation theory. We can interpret such modules in terms of the 2-monad $\mathtt{disc}(M)\times -$ on the 2-category $\Cat$.
  
  Here $\mathtt{disc}(M)$ is the category whose objects are elements of $M$, and whose morphisms are just the identity arrows. Since $M$ is a monoid, $\mathtt{disc}(M)$ is a monoidal category, so $\mathtt{disc}(M)\times -$ is a 2-monad.

As the set of monoid endomorphisms on a monoid $A$ is isomorphic to the set of endofunctors on the category given by the delooping of the said monoid, an $M$-module can be seen to be a one-object algebra for this 2-monad. Indeed, this is just a special case of an actegory.
\end{example}

\begin{example}[Cocycles as lax morphisms]
  The monoid cocycles, or ``crossed homomorphisms'', used in \cite{dudzik2023asynchronous} to study asynchrony in algorithms and networks, can be described as lax morphisms for 2-monad algebras. As above, if $A$ is an $M$-module, then $\cB A$ is an algebra for $\mathtt{disc}(M)\times -$. A lax morphism $1\to \cB A$ is a natural transformation of the unique functor $M\to \cB A$, which is equivalently just a set map $D: M\to A$. The axiom for compositionality is exactly the (right) 1-cocycle condition for $D$, so 1-cocycles are the same as lax morphisms $1\to \cB A$.
\end{example}


\section{The 2-category $\Para(\cC)$}\label{sec:2-categories_para(c)}

\begin{definition}[{Para, compare \cite{cruttwell_categorical_2022,capucci_towards_2022} }]
  Given a monoidal category $(\cM,\otimes,I)$ and an $\cM$-actegory $\cC$, let $\Para_{\act}(\cC)$ be the $2$-category whose:
  \begin{itemize}
      \item Objects are the objects of $\cC$;
      \item Morphisms $X \rightarrow Y$ are pairs $(P,f)$, where $P$ is an object of $\cM$ and $f:P \act X \rightarrow Y$ is a morphism of $\cC$;
      \begin{figure}[H]
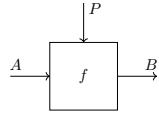

        \scaletikzfig[0.6]{para}
        \caption{String diagram representation of a parametric morphism. We often draw the parameter wire on the \emph{vertical} axis to signify that the parameter object is part of the data of the morphism.}
        \label{fig:app_para_morphism}
      \end{figure}
      \item $2$-morphisms $(P,f) \Rightarrow (P^\prime,f^\prime):X \rightarrow Y$ are $1$-morphisms $r:P^\prime \rightarrow P$ of $\cM$ such that the triangle
      
      
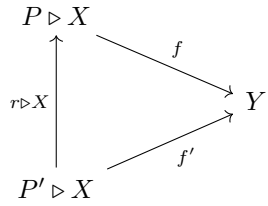
\[\begin{tikzcd}[ampersand replacement=\&]
	{P\triangleright X} \\
	\&\& Y \\
	{P^\prime\triangleright X}
	\arrow["{r\triangleright X}", from=3-1, to=1-1]
	\arrow["{f^\prime}"', from=3-1, to=2-3]
	\arrow["f", from=1-1, to=2-3]
      \end{tikzcd}\]

      commutes (equivalently an equality of parametric string diagrams as in \cref{fig:para_reparam}).
      
      \begin{figure}[H]
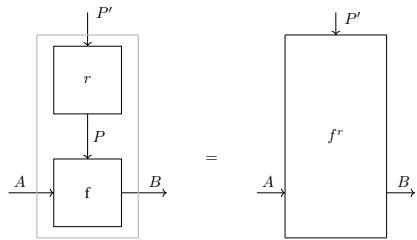

        \scaletikzfig[0.6]{para_reparametrisation}
        \caption{String diagram of reparameterisation. The reparameterisation map $r$ is drawn vertically.}
        \label{fig:para_reparam}
      \end{figure}
      \item Identity morphism on $X$ is the parametric map $(I, {\eta_X}^{-1})$, where $\eta$ is the unitor of the underlying actegory.
  \end{itemize}

  where composition of morphisms $(P,f):X \rightarrow Y$ and $(Q,g):Y \rightarrow Z$ is $(Q \otimes P, h)$ where $h$ is the composite
    \[
        (Q \otimes P) \act X \xrightarrow{\mu_{Q,P,X}} Q \act (P \act X) \xrightarrow{Q \act f} Q \act Y \xrightarrow{g} Z
    \]
  and composition of 2-morphisms is given by the composition of $\cM$.
    \begin{figure}[H]
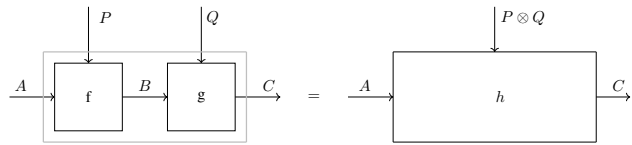

      \scaletikzfig[0.6]{para_composition}
      \caption{String diagram representation of the composition of parametric morphisms. By treating parameters on a separate axis we obtain an elegant graphical depiction of their composition.}
      \label{fig:para_composition}
    \end{figure}
\end{definition}

\begin{remark}
    When the actegory is strict the identities of the actegory reduce to being the parametric morphisms of the form $(I,\id_X)$
\end{remark}

\begin{example}[Real Vector Spaces and Smooth Maps]
  Consider the cartesian category $\Smooth$ whose objects are real vector spaces, and morphisms are smooth functions.
  As this category is cartesian, we can form $\Para(\Smooth)$ modelling parametric smooth functions.
\end{example}

\begin{lemma}[{Embedding of $\cC$ into $\Para(\cC)$}]
  \label{lemma:embedding}
  There is a 2-functor $\gamma: \cC \to \Para(\cC)$ which is identity-on-objects and treats every morphism $f : A \to B$ as trivially parametric, i.e.\ as $(I, f \circ \eta^{-1}_A)$.\footnote{Here we are treating $\cC$ as 2-category with only identity 2-morphisms.}
\end{lemma}

\begin{remark}
    When the actegory $(\cC,\act)$ is strict the trivially parametric functions do not require precomposition with the unitor of actegory. That is to say a morphism $f:X\rightarrow Y$ of $\cC$ is sent to the parametric morphism $(I,f):X \rightarrow Y$.
\end{remark}

\begin{theorem}[{$\gamma$ preserves connected colimits}]
  Suppose $(\cM,\eta,\mu)$ is a strict monoidal category, that $(\cC,\act)$ is an $\cM$-actegory, and that $X\act (\_):\cC \rightarrow \cC $ preserves connected colimits. Then, the 2-functor $\gamma:\cC \rightarrow \Para(\cC)$ preserves connected colimits.
\end{theorem}

\begin{remark}
    While in $\cC$ they are the only kind of colimits, the colimits we mean in $\Para(\cC)$ are strict colimits. The relationship to other flavours of colimit is a topic of ongoing research. 
\end{remark}

\begin{proof}
    Note that for a cone over a trivially parameterised connected diagram to commute, the parameters of the 1-cells of the cone must all be the same. Then, since $\act$ preserves connected colimits by hypothesis, we may assemble the following chain of isomorphisms.

\begin{eqnarray*}
\mathsf{Tr}_{1}\left(\mathsf{Para}\left(\mathcal{C}\right)\right)\left(\underrightarrow{\mathsf{lim}}_{d\in\mathcal{D}}X_{d},Y\right) & \overset{\sim}{\rightarrow} & \coprod_{P\in\mathcal{M}}\mathcal{C}\left(P\triangleright\underrightarrow{\mathsf{lim}}_{d\in\mathcal{D}}X_{d},Y\right)\\
 & \overset{\sim}{\rightarrow} & \coprod_{P\in\mathcal{M}}\mathcal{C}\left(\underrightarrow{\mathsf{lim}}_{d\in\mathcal{D}}P\triangleright X_{d},Y\right)\\
 & \overset{\sim}{\rightarrow} & \coprod_{P\in\mathcal{M}}\underleftarrow{\mathsf{lim}}_{d\in\mathcal{D}}\mathcal{C}\left(P\triangleright X_{d},Y\right)\\
 & \overset{\sim}{\rightarrow} & \mathsf{\mathsf{Ob}\left(\mathcal{M}\right)\times\underleftarrow{\mathsf{lim}}_{d\in\mathcal{D}}\mathcal{C}\left(P\triangleright X_{d},Y\right)}\\
 & \overset{\sim}{\rightarrow} & \underleftarrow{\mathsf{lim}}_{d\in\mathcal{D}}\left(\mathsf{Ob}\left(\mathcal{M}\right)\times\mathcal{C}\left(P\triangleright X_{d},Y\right)\right)\\
 & \overset{\sim}{\rightarrow} & \underleftarrow{\mathsf{lim}}_{d\in\mathcal{D}}\left(\coprod_{P\in\mathcal{M}}\mathcal{C}\left(P\triangleright X_{d},Y\right)\right)\\
 & \overset{\sim}{\rightarrow} & \underleftarrow{\mathsf{lim}}_{d\in\mathcal{D}}\left(\mathsf{Tr}_{1}\left(\mathsf{Para}\left(\mathcal{C}\right)\right)\left(X_{d},Y\right)\right)
\end{eqnarray*}

\end{proof}


\begin{example}
  \label{prop:para_pseudomonad}
  Fix $\cM$-actegory $(\cC,\act)$, and a monad $(T,\mu,\eta)$ on $\cC$ with actegorical strength $\sigma: M \act T(X) \to T(M \act X)$.
  Then this gives us a 2-monad on $\Para_{\act}(T)$ on $\Para_{\act}(\cC)$ whose underlying 2-endofunctor:
  \begin{itemize}
      \item Acts as $T$ does on objects\footnote{This is well defined since $\Ob{\Para(\cC)} = \Ob{\cC}$}; and
      \item Sends $(P,f):X \rightarrow Y$ to $(P, f')$, where $f'$ is the composite
      \[
        P \act T(X) \xrightarrow{\sigma_{P, X}} T(P\act X) \xrightarrow{T(f)} T(Y)
      \]
      \item Sends $r : P \to P'$ to itself.
  \end{itemize}
  and the unit and multiplication 2-natural transformations are defined as follows:
    \begin{itemize}
        \item Unit $\Para_{\act}(\eta) : \id_{\Para_{\act}(\cC)} \Rightarrow \Para_{\act}(T)$ whose component at $X : \cC$ is the element of $\Para_{\act}(\cC)(X, T(X))$ given by the parametric morphism $(I, u_X)$, where $u_X$ is the composite
        \[
            I \act X \xrightarrow{{\eta^{\act}_X}^{-1}} X \xrightarrow{\eta^T_X} T(X)
        \]
        For every parametric map $(P, f)\in\Para(\cC)(X, Y)$ we have the strictly natural square in \cref{eq:para_unit_naturality_square}.
        The 2-cell is given by the reparameterisation $\beta_{P, I} : P \otimes I \to I \otimes P$ which is identity since we are assuming $\cM$ is strict monoidal.
        It can be checked that this indeed satisfies the conditions of a 2-morphism in $\Para_{\act}(\cC)$.
        This makes $\Para_{\act}(\eta)$ a 2-natural transformation.
        \begin{equation}
            \label{eq:para_unit_naturality_square}
\begin{tikzcd}[ampersand replacement=\&]
	X \&\& Y \\
	\\
	{T(X)} \&\& {T(Y)}
	\arrow["{T(f) \circ \sigma_{P,X}}"', from=3-1, to=3-3]
	\arrow["{(I, u)}"', from=1-1, to=3-1]
	\arrow["{(P, f)}", from=1-1, to=1-3]
	\arrow["{(I, u_Y)}", from=1-3, to=3-3]
	\arrow["{\beta_{P,I}}"', shorten <=11pt, shorten >=11pt, Rightarrow,  to=1-3, from=3-1]
\end{tikzcd}
        \end{equation}
        \item Multiplication $\Para_{\act}(\mu) : \Para_{\act}(T)^2 \Rightarrow \Para_{\act}(T)$ whose component at $X\in\Para_{\act}(\cC)$ is the parametric morphism $(I, m_X)$ where $m_X$ is the composite
        \[
            I \act T^2(X) \xrightarrow{{\eta^{\act}_X}^{-1}} T^2(X) \xrightarrow{\mu^T_X} T(X)
        \]
        For every parametric morphism $(P, f)\in\Para_{\act}(\cC)(X, Y)$ we have the strictly natural square in \cref{eq:para_mult_naturality_square}.
        It is again given by $\beta_{P, I}$, and is strict because $\cM$ is strict monoidal.
        It can be checked that this too satisfies the conditions of a 2-morphism in $\Para_{\act}(\cC)$
        This makes $\Para_{\act}(\mu)$ a 2-natural transformation.
        \begin{equation}\label{eq:para_mult_naturality_square}
\begin{tikzcd}[ampersand replacement=\&]
	{T^2(X)} \&\&\&\& {T^2(Y)} \\
	\\
	{T(X)} \&\&\&\& {T(Y)}
	\arrow["{T(f) \circ \sigma_{P, X}}"', from=3-1, to=3-5]
	\arrow["{(I, m_x)}"', from=1-1, to=3-1]
	\arrow["{(I, m_Y)}", from=1-5, to=3-5]
	\arrow["{(P, (T(f) \circ \sigma_{P, X}) \circ \sigma_{P, T(X)})}", from=1-1, to=1-5]
	\arrow["{\beta_{P, I}}"', shorten <=33pt, shorten >=33pt, Rightarrow, to=1-5, from=3-1]
\end{tikzcd}
        \end{equation}
    \end{itemize}
    Lastly, we need to check that this indeed satisfies the 2-monad coherence conditions.
    Since the parameters of the components of $\Para_{\act}(\eta)$ and $\Para_{\act}(\mu)$ are all trivial, this becomes straightforward to check.
\end{example}

\begin{remark}
    For strictly monoidal $\cM$ and a strict actegory $(\cC,\act)$ the unit and multiplication of $\Para(T)$ are exactly the unit and multiplication of $T$.
\end{remark}

The following theorem shows us how out of the abstract 2-categorical framework the notion of weight tying (usage of the same weight in two different places obtained by copying the value) arises automatically as the structure of a comonoid induced by a lax algebra of a strong actegorical monad on $\cC$.

\begin{restatable}[Lax (co)algebras for $\Para_{\act}(T)$ induce comonoids]{theorem}{ComonoidsLaxAlgebras}
  \label{thm:parameters_lax_algebra_comonoids}
  Let $(\cC, \act)$ be a $\cM$-actegory and $T : \cC \to \cC$ a strong actegorical monad on $\cC$.
  Consider a lax algebra $(A, (P, a), \epsilon_A, \delta_A)$ for the induced 2-monad $\Para(T)$.
  Then $P$ is a comonoid in $\cM$ where $\epsilon_A$ is the data of its counit, and $\delta_A$ the data of its comultiplication, and the comonoid laws follow from lax algebra coherence conditions.
  Dually, the same statement holds for a lax coalgebra of a $\Para(T)$ comonad.
\end{restatable}

\begin{proof}
    We start by unpacking the data of the lax algebra 2-cells.

\[\begin{tikzcd}[ampersand replacement=\&]
	\& {T(X)} \&\&\&\& {T(T(X))} \&\& {T(X)} \\
	\\
	X \&\& X \&\&\& {T(X)} \&\& X
	\arrow[""{name=0, anchor=center, inner sep=0}, "{(I, \lambda_X)}", Rightarrow, no head, from=3-3, to=3-1]
	\arrow["{(I, \eta^T_X \circ {\eta^{\act}_X}^{-1})}"'{pos=0.1}, from=3-3, to=1-2]
	\arrow["{(P, a)}"', from=1-2, to=3-1]
	\arrow["{(I, m_x)}", from=1-6, to=1-8]
	\arrow["{(P, T(a) \circ \sigma_{P,X})}"', from=1-6, to=3-6]
	\arrow["{(P, a)}"', from=3-6, to=3-8]
	\arrow["{(P, a)}", from=1-8, to=3-8]
	\arrow["{\delta_P}", shorten <=12pt, shorten >=12pt, Rightarrow, from=3-6, to=1-8]
	\arrow["{\epsilon_P}", shorten <=7pt, shorten >=7pt, Rightarrow, from=0, to=1-2]
\end{tikzcd}\]

We can see that they are given by two reparameterisations: $\epsilon_P : P \otimes I \to I$ and $\delta_P : I \otimes P \to P \otimes P$.
These uniquely determine the morphisms which we call $!_P : P \to I$ and $\Delta_P : P \to P \otimes P$ respectively.
To show that this is indeed a comonoid, we need to unpack the lax algebra laws.
We note that all the conditions unpack only to conditions on morphisms in $\cM$.
It is relatively tedious, but straightforward to check that these conditions ensure that $!_P$ and $\Delta_P$ satisfy the comonoid laws.
\end{proof}

\begin{corollary}
    The functor $\AlgLaxMnd{\Para(T)} \rightarrow \mathsf{CoMon}(\cC) \times \mathsf{Lax}(\rightarrow,\cC)$ which takes lax algebras for $\Para(T)$ to the pair of the underlying comonoid parameter and the $\cC$-morphism which interprets the parametric structure map is full-and-faithful.
\end{corollary}

\begin{proof}
    The proof is formal and left to the reader.
\end{proof}

\begin{conjecture}\label{equivalence_of_lax_algebras}
    Given a cartesian monoidal category $(\cM,\times,1)$, an $\cM$-actegory $(\cC,\act)$, and a strong endofunctor $F$ on $\cC$ such that $F$ and $\cC$ satsify the hypotheses of \cref{Kelly_unified_applied_to_well_pointed_endo}, then the assignment of a lax algebra $(X, (P, f))$ to
    \[
    (X,((P,!_P:P \rightarrow 1,\Delta_P : P \rightarrow P \times P),\underrightarrow{\mathsf{laxlim}}((P,f)\circ \Para(F)(P,f)\circ \cdots \circ \Para(F)^\alpha(P.f)):F^\kappa X \rightarrow X))\] defines an equivalence of categories \[\AlgLaxEndo{\Para(F)} \rightarrow \AlgLaxMnd{F^\kappa}\] 
\end{conjecture}

\section{Weight Tying Examples}

\subsection{Examples from Geometric Deep Learning}\label{app:gdl_examples}

\begin{example}[Linear equivariant layers for a pair of pixels]
Consider the category $\Vect$ of finite-dimensional vector spaces and linear maps.
For simplicity, we will assume that the carrier set for our data is $R^{\Z_2}$, which is a pair of pixels. Consider a linear endofunction on such data, $f_{\bf W} : \R^{\Z_2} \to \R^{\Z_2}$. It is well known that this function can be represented as a multiplication with a $2\times 2$ matrix ${\bf W} = \begin{bmatrix}
        w_1 & w_3 \\
        w_2 & w_4
    \end{bmatrix}$.

Now consider the group of 1D translations $(\Z_2, +, 0)$---which in this case amounts to \emph{pixel swaps}---and the induced action $\act$ on $\R^{\Z_2}$.
If $f:\R^{\Z_2} \rightarrow \R^{\Z_2}$ is equivariant with respect to $\act$, then via \cref{eq:equivariance}, for any input $[x_1, x_2]\in\R^{\Z_2}$ it must hold that
\begin{equation}
    \begin{bmatrix}
        w_1x_2 + w_2x_1 \\
        w_3x_2 + w_4x_1
    \end{bmatrix}
    =
    \begin{bmatrix}
        w_3x_1 + w_4x_2 \\
        w_1x_1 + w_2x_2 
    \end{bmatrix}
\end{equation}

This implies that $w_3 = w_2$ and $w_4 = w_1$, meaning every $\bf W$ is of the form
 
\[
\begin{bmatrix}
    \textcolor{blue}{w_1} & \textcolor{red}{w_2} \\
    \textcolor{red}{w_2} & \textcolor{blue}{w_1}
\end{bmatrix}
\]

where the weight tying makes the matrix a symmetric one. Any neural network $f_{\bf W}$ satisfying this constraint will be a \emph{linear translation equivariant layer} over $\Z_2$, and can be used as a building block to construct geometric deep learning architectures.
\end{example}
\begin{remark}[Circulants and CNNs]
Note that this concept generalises to larger input domains (e.g. $\mathbb{R}^{\Z_k}$ for $k > 2$). Generally, it is a well-known fact in signal processing that, for $f_{\bf W}$ to satisfy a linear translation equivariance constraint, $\mathbf{W}$ must be a \emph{circulant} matrix. Circulant matrices are known in neural networks as \emph{convolutional layers}, the essence of modern CNNs \citep{fukushima1983neocognitron,lecun1998gradient}.
\end{remark}

\begin{example}[Invariant maps]
  \label{ex:invariance}
  Invariant maps are also $(G \times -)$-algebra homomorphisms where the codomain is a trivial group action.
  Specifically, setting the domain to any of the group actions from \cref{eq:equivariance}, the corresponding induced commutative diagram is below:
\[\begin{tikzcd}[ampersand replacement=\&]
	{G \times \R^{\Z_w \times \Z_h}} \& {G \times  \R^{\Z_w \times \Z_h}} \\
	{\R^{Z_w \times \Z_h}} \& {\R^{\Z_w \times \Z_h}}
	\arrow["\act"', from=1-1, to=2-1]
	\arrow["{\pi_2}", from=1-2, to=2-2]
	\arrow["f"', from=2-1, to=2-2]
	\arrow["{G \times f}", from=1-1, to=1-2]
\end{tikzcd}\]

 Elementwise, the commutativity of that diagram unpacks to the equation 
 \[
     f(g \act x) = f(x)
 \]
 The translation example, for instance, recovers the equation $f(((i', j') \act x)(i, j)) = (i', j') \act f(x)(i, j)$ which reduces to  $f(x(i - i', j - j')) = f(x)(i, j)$.
 Intuitively, it states that, for any displacement $(i', j')$, the result of translating the input by $(i', j')$ and then applying the function $f$ is the same as just applying $f$.
\end{example}

We can run the analogous calculation with weight sharing here.
We consider the same linear endofunction $f_{\bf W}\in\R^{\Z_2} \to \R^{\Z_2}$ on a pair of pixels (represented as a matrix $W = \begin{bmatrix}
        w_1 & w_3 \\
        w_2 & w_4
    \end{bmatrix}$), and the same group of translations $(\Z_2, +, 0)$ on the domain, but this time we consider the identity endomorphism on the codomain.
If $f_{\bf W}$ is invariant to $\act$, then \cref{ex:invariance} for any input $[x_1, x_2]\in\R^{\Z_2}$ it has to hold that
\begin{equation}
    \begin{bmatrix}
        w_1x_2 + w_2x_1 \\
        w_3x_2 + w_4x_1
    \end{bmatrix}
    =
    \begin{bmatrix}
        w_1x_1 + w_2x_2 \\
        w_3x_1 + w_4x_2
    \end{bmatrix}
\end{equation}

implying $w_1 = w_2$ and $w_3 = w_4$, meaning the equivariance induced a particular weight tying scheme
 
\[
\begin{bmatrix}
    \textcolor{blue}{w_1} & \textcolor{red}{w_3} \\
    \textcolor{blue}{w_1} & \textcolor{red}{w_3}
\end{bmatrix}
\]

in which the matrix has shared rows. Note that this implies that each of the pixel values in the output of $f_{\bf W}$ will be the same, and hence this layer could also be represented as just a single dot product with $\begin{bmatrix}\textcolor{blue}{w_1} & \textcolor{red}{w_3}\end{bmatrix}$, \emph{eliminating} the grid structure in the original set. Such a layer would not be an endofunction on $\R^{\Z_2}$, however, which was our initial space of exploration.

\subsection{Examples from Automata Theory}
\label{ex:automata_theory}

\subsubsection{Streams}

\begin{example}[Streams]
  \label{ex:stream_as_coalgebra}
  Let $O$ be a set, thought of as the set of outputs.
  Consider the endofunctor from $O \times - : \Set \to \Set$.
  Then the set $\Stream{O}$ of streams\footnote{An infinite sequence of outputs, isomorphic to the set $(\N \to O)$ of functions from the natural numbers into $O$.} with outputs $O$, together with the map $\product{\StreamOutput}{\StreamNext} : \Stream{O} \to O \times \Stream{O}$ forms a coalgebra of this endofunctor.
  They have a representation as the following datatype
\begin{code}
data Stream o = MkStream {
  output :: o
  next :: Stream o
}
\end{code}
This datatype describes streams \emph{coinductively}, as something from which we can always extract an output, and another stream.
In \cref{ex:unfolding_rnn_cell} we will see how this will be related to unfolding recurrent neural networks.
\end{example}

\begin{example}[Unfolds to streams as coalgebra homomorphisms]
  \label{ex:streams_terminal}
  Consider the endofunctor $(O \times -)$ from \cref{ex:stream_as_coalgebra}, and a coalgebra homomorphism from any other $(O \times -)$-coalgebra $(X, \product{o}{n})$ into $(\Stream{O}, \product{\StreamOutput}{\StreamNext}$:

\[\begin{tikzcd}[ampersand replacement=\&]
	X \&\& {\Stream{O}} \\
	{O \times X} \&\& {O \times \Stream{O}}
	\arrow["{\langle o, n \rangle}"', from=1-1, to=2-1]
	\arrow["{f_{o, n}}", from=1-1, to=1-3]
	\arrow["{O \times f_{o, n}}"', from=2-1, to=2-3]
	\arrow["{\langle \StreamOutput, \StreamNext \rangle}", from=1-3, to=2-3]
\end{tikzcd}\]

Then the map $f_{o, n} : X \to \Stream{O}$ is necessarily a \emph{unfold} to a stream, a concept from functional programming describing how a stream of vaues can be obtained from a single value.
It is implemented by \emph{corecursion} on the input:

\begin{code}
@$f_{o, n}$@ :: x -> Stream o
@$f_{o, n}$@ x = MkStream (o x) (@$f_{o, n}$@ (n x))
\end{code}

This corecursion is structural in nature, meaning it satisfies the following two equations which arise by unpacking the algebra homomorphism equations elementwise:
\begin{align}
  o(x) = \StreamOutput(f_{o, n}(x)) \label{eq:stream_equation1}\\
  f_{o, n}(n(x)) = \StreamNext(f_{o, n}(x)) \label{eq:stream_equation2}
\end{align}

Here \cref{eq:stream_equation1} tells us that the output of the stream produced by $f_{o, n}$ at $x$ is $o(x)$, and \cref{eq:stream_equation2} tells us that the rest of the stream is the stream produced by $f_{o, n}$ at $x$ is $f_{o, n}(n(x))$.
\end{example}

\begin{remark}
  \label{rem:streams_terminality}
  Analogous to \cref{rmk:freemonoid}, there can only ever be \emph{one} unfold of this type.
  This is because streams are a terminal object in the category of $(O \times -)$-coalgebras.
\end{remark}

We can study the weight sharing induced by streams.
Consider the category $\Vect$ and the coreader comonad $\R \times -$ given by the output set $\R$.
Fix a coalgebra $(\R, \product{o}{n})$, and represent the coalgebra map with scalars $w_o$ and $w_n$, denoting the output and the next state, respectively.
Then the universal coalgebra homomorphism is a linear map $f_{o, n} : \R \to \Stream{\R}$ satisfying \cref{eq:stream_equation1,eq:stream_equation2}.
If we represent this as an infinite-dimensional matrix
\[
\begin{bmatrix}
      w_1 & w_2 & w_3 & w_4 & \dots
\end{bmatrix}
\]

then the induced weight sharing scheme removes any degrees of freedom: the matrix is completely determined by $w_o$ and $w_n$, and is of the form:
\[
\begin{bmatrix}
      \textcolor{blue}{w_o} & \textcolor{red}{w_n}\textcolor{blue}{w_o} & \textcolor{red}{w_n^2}\textcolor{blue}{w_o} & \textcolor{red}{w_n^3}\textcolor{blue}{w_o} & \dots
\end{bmatrix}
\]


\subsubsection{Moore Machines.}

\begin{example}[Moore machines]
  \label{ex:moore_machine_as_coalgebra}
  Let $O$ and $I$ be sets, thought of as sets of outputs and inputs, respectively.
  Consider the endofunctor $ O \times (I \to -) : \Set \to \Set$.
  Then the set $\Moore{O}{I}$ of Moore Machines with $O$-labelled outputs and $I$-labelled inputs together with the map $\product{\StreamOutput}{\StreamNext} : \Moore{O}{I} \to O \times (I \to \Moore{O}{I})$ forms a coalgebra for this endofunctor.
  They can be represented as the following datatype.
  \begin{code}
    data Moore o i = MkMoore {
      output :: o,
      nextStep :: (i -> Moore o i)
    }
  \end{code}
  Like with Mealy machines, this description is coinductive.
  From a Moore machine we can always extract an output and a function which given an input produces another Moore machine.
  In \cref{ex:moore_cell} we will see how this will be related to general recurrent neural networks of a particular form.
\end{example}

\begin{example}[Unfolds to Moore machines are terminal $(O \times -)$-coalgebras]
  \label{ex:moore_terminal}
  Consider the endofunctor $O \times (I \to -)$ from \cref{ex:moore_machine_as_coalgebra}, and any other $(O \times -)$-coalgebra $(X, \product{o}{n})$ to $(\Moore{O}{I}, \product{\StreamOutput}{\StreamNext}$:

\[\begin{tikzcd}[ampersand replacement=\&]
	X \&\& {\Moore{O}{I}} \\
	{O \times (I \to X)} \&\& {O \times (I \to \Moore{O}{I})}
	\arrow["{\product{o}{n}}"', from=1-1, to=2-1]
	\arrow["{\product{\StreamOutput}{\StreamNext}}", from=1-3, to=2-3]
	\arrow["{f_{o, n}}", from=1-1, to=1-3]
	\arrow["{O \times (I \to f_{o, n})}"', from=2-1, to=2-3]
\end{tikzcd}\]
Then the map $f_{o, n}$ is a necessarily an \emph{unfold} to a Moore machine, a function describing how a Moore machine is obtained from a single value.
It is a corecursive function as below:

\begin{code}
@$f_m$@ :: x -> Moore o i
@$f_m$@ x = MkMoore (@$m_1$@ x) (@$\lambda$@i @$\mapsto$@ @$f_m$@ (@$m_2$@ x i))
\end{code}

which is structural in nature, meaning it satisfies the following two equations which arise by unpacking the coalgebra homomorphism equations elementwise:
\begin{align}
  o(x) = \StreamOutput(f_{o, n}(x)) \label{eq:moore_equation1}\\
  f_{o, n}(n(x)(i)) = \StreamNext(f_{o, n}(x))(i) \label{eq:moore_equation2}
\end{align}
Here \cref{eq:moore_equation1} tells us that the output of the Moore machine produced by $f_{o, n}$ at state $x$ is given by the output of $n$ at state $x$, and \cref{eq:moore_equation2} tells us that the next Moore machine produced at $x$ and $i$ is the one produced by $f_{o, n}$ at $n(x)(i)$.
\end{example}


\section{Parametric 2-endofunctors and their Algebras}
\label{sec:parametric_endofunctor_algebras}

Just like before, often at our disposal we will have a more minimal structure.
Instead of requiring a strong 2-monad, we can instead require merely a strong 2-endofunctor.
The data of an algebra for a 2-endofunctor will be the same as the algebra for an endofunctor \cref{def:algebra_for_endofunctor}.

\begin{example}[Folding RNN cell]
  \label{ex:folding_rnn_cell}
  Consider the endofunctor $1 + A \times -$ from \cref{ex:list_as_algebra}.
  Via strength from \cref{ex:strong_endofunctors} we can form the 2-endofunctor $\Para(1 + A \times -) : \Para(\Set) \to \Para(\Set)$.
  Then a folding recurrent neural network arises as its algebra.

  More concretely, an algebra here consists of the carrier set $S$ (name suggestively chosen to denote \emph{hidden state}) and a parametric map $(P, \foldingRecurrentCell)\in\Para(\Set)(1 + A \times S, S)$.
  Via the isomorphism $P \times (1 + A \times S) \cong P + P \times A \times S$ we can break $\foldingRecurrentCell$ into two pieces: the choice of the initial hidden state $\foldingRecurrentCell_0 : P \to X$ and the folding recurrent neural network cell $\foldingRecurrentCell_1 : P \times A \times X \to X$, as shown in the figure below.
  \begin{figure}[h]
      \begin{subfigure}[c]{0.5\textwidth}
      \[\begin{tikzcd}[ampersand replacement=\&]
      	{1 + A \times S} \\
      	S
      	\arrow["{(P, \foldingRecurrentCell)}", from=1-1, to=2-1]
      \end{tikzcd}\]
      \end{subfigure}%
      \hfill
      \begin{subfigure}[c]{0.5\textwidth}
        \scaletikzfig[0.8]{folding_rnn_cell}
        \label{fig:folding_rnn_cell}
      \end{subfigure}%
      \caption{An algebra for $\Para(1 + A \times -)$ consists of a recurrent neural network cell and an initial hidden state.}
      \label{fig:folding_rnn_cell_algebra}
    \end{figure}
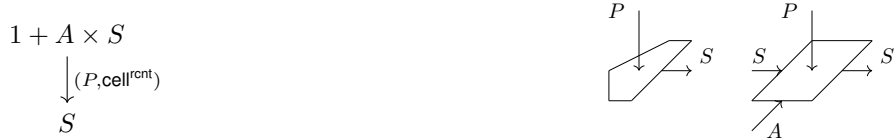
  We will see how iterating this construction will produce a folding recurrent neural network which consumes a sequence of inputs and iteratively updates its hidden state.\footnote{We note here that the hidden state is allowed to depend on the parameter, something which is not possible in the usual definition of a recurrent neural network cell.}
\end{example}

\begin{example}[Recursive NN cell]
  \label{ex:recursive_cell}
  Consider the endofunctor $A + (-)^2$ from \cref{ex:tree_as_algebra}.
  Via strength from \cref{ex:strong_endofunctors} we can form the 2-endofunctor $\Para(A + (-)^2)$ on $\Para(\Set)$.
  Then a recursive neural network arises as its algebra.

  More concretely, an algebra here consists of the carrier set $S$ and a parametric map $(P, \recursiveCell)\in\Para(\Set)(A + S^2, S)$.
  Via the isomorphism $P \times (A + S^2) \cong P + P \times A \times S^2$ we can break $\recursiveCell$ into two pieces: the choice of the initial hidden state $\recursiveCell_0 : P \to S$ and the recursive neural network cell $\recursiveCell_1 : P \times A \times S^2 \to S$, as shown in the figure below.

  \begin{figure}[h]
      \begin{subfigure}[c]{0.5\textwidth}
    \[\begin{tikzcd}[ampersand replacement=\&]
    	{A + S^2} \\
    	S
    	\arrow["{(P, \recursiveCell)}", from=1-1, to=2-1]
    \end{tikzcd}\]
      \end{subfigure}%
      \hfill
      \begin{subfigure}[c]{0.5\textwidth}
        \scaletikzfig[0.8]{recursive_cell_lab-leaves_no-nodes}
        \label{fig:recursive_cell}
      \end{subfigure}%
      \caption{An algebra for $\Para(A + (-)^2)$ consists of a recursive neural network cell and an initial hidden state.}
      \label{fig:recursive_cell_algebra}
    \end{figure}
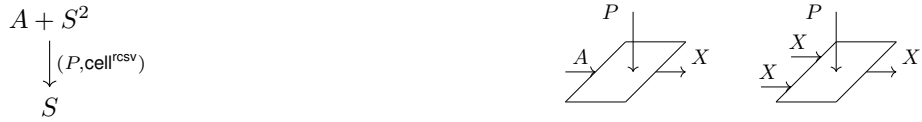
\end{example}
\begin{example}[Unfolding RNN cell]
  \label{ex:unfolding_rnn_cell}
  Consider the endofunctor $O \times -$ from \cref{ex:stream_as_coalgebra}.
  Via strength $\sigma_{P, X} :  P \times (O \times X) \xrightarrow{\cong} O \times P \times X$ we can form the 2-endofunctor $\Para(O \times -) : \Para(\Set) \to \Para(\Set)$.
  Then an unfolding recurrent neural network arises as its coalgebra.

  More concretely, a coalgebra here consists of the carrier set $S$, and a parametric map $(P, \unfoldingRecurrentCell)\in\Para(\Set)(S, O \times S)$.
  Here $\unfoldingRecurrentCell$ consists of maps $\cell_o : P \times S \to O $ which computes the output and $\cell_n : P \times S \to S$ which computes the next state.\footnote{By universal property of the product we can treat them as one cell of type $P \times S \to O \times S$ or as two cells.}
  \begin{figure}[h]
      \begin{subfigure}[c]{0.5\textwidth}
       \[\begin{tikzcd}[ampersand replacement=\&]
       	S \\
       	{O \times S}
       	\arrow["{(P, \unfoldingRecurrentCell)}", from=1-1, to=2-1]
       \end{tikzcd}\]
      \end{subfigure}%
      \hfill
      \begin{subfigure}[c]{0.5\textwidth}
        \scaletikzfig[0.8]{unfolding_recurrent_cell}
        \label{fig:unfolding_rnn_cell}
      \end{subfigure}%
      \caption{An algebra for the $\Para(O \times -)$ 2-endofunctor consists of an unfolding recurrent neural network cell.}
      \label{fig:unfolding_rnn_cell_algebra}
    \end{figure}
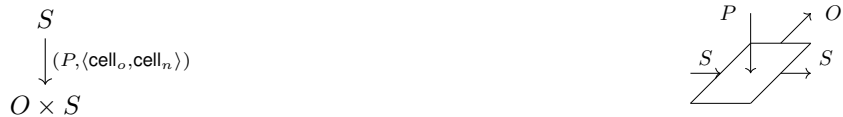
  We will see how iterating this construction will produce an unfolding recurrent neural network which given a starting state produces a stream of outputs $O$. 
\end{example}

\begin{example}[Mealy machine cell / Full RNN cell]
  \label{ex:mealy_cell}
  Consider the endofunctor $I \to O \times -$ from \cref{ex:mealy_machine_as_coalgebra}.
  Via strength from \cref{ex:strong_endofunctors} we can form the 2-endofunctor $\Para(I \to O \times -) : \Para(\Set) \to \Para(\Set)$.
  Then a Mealy machine cell arises as its coalgebra.

  More concretely, a coalgebra here consists of the carrier set $S$ and a parametric map $(P, \mealyCell)\in\Para(\Set)(S, I \to O \times S)$.
  Here $\mealyCell$ can be thought of as a map $P \times S \to I \to O \times S$, as shown in the figure below.
  \begin{figure}[h]
      \begin{subfigure}[c]{0.5\textwidth}
      \[\begin{tikzcd}[ampersand replacement=\&]
      	S \\
      	{(I \to O \times S)}
      	\arrow["{(P, \mealyCell)}", from=1-1, to=2-1]
      \end{tikzcd}\]
      \end{subfigure}%
      \hfill
      \begin{subfigure}[c]{0.5\textwidth}
        \scaletikzfig[0.8]{mealy_cell}
        \label{fig:mealy_cell}
      \end{subfigure}%
      \caption{An algebra for $\Para((I \to O \times -)$ 2-endofunctor consists of an object $S$ and a map $f : P \times S \times I \to O \times S$ which we have taken the liberty to uncurry.}
      \label{fig:mealy_cell_algebra}
    \end{figure}
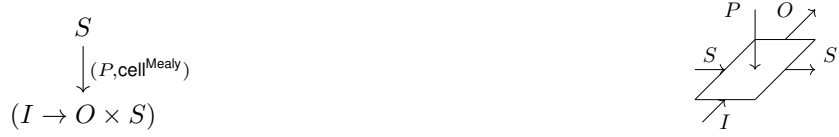
    We interpret it as a full recurrent neural network consuming a hidden state $S$, input $I$ and producing an output $O$ and an updated hidden state $S$.
\end{example}
This suggests that recurrent neural networks can be thought of as learnable Mealy machines, a perspective seldom advocated for in the literature.
\begin{example}[Moore machine cell]
  \label{ex:moore_cell}
  Consider the endofunctor $O \times (I \to -)$ from \cref{ex:moore_machine_as_coalgebra}.
  Via strength $\sigma_{P, X} : P \times O \times (I \to O) \to O \times (I \to P \times X)$ defined as $(p, o, f)\mapsto (o, \lambda i \mapsto (p, f(i)))$ we can form the 2-endofunctor $\Para(O \times (I \to -)) : \Para(\Set) \to \Para(\Set)$.
  Then a Moore machine cell arises as its coalgebra.

  More concretely, a coalgebra here consists of the carrier set $S$ and a parametric map $(P, \mooreCell)\in\Para(\Set)(S, O \times (I \to S))$.
  Here $\mooreCell$ can be thought of as a map $P \times S \to O \times (I \to S)$, as shown in the figure below.
  We can break it down into two pieces $\mooreCell_o : P \times S \to O$ and $\mooreCell_n : P \times S \times I \to S$.

  \begin{figure}[h]
      \begin{subfigure}[c]{0.5\textwidth}
      \[\begin{tikzcd}[ampersand replacement=\&]
      	X \\
      	{O \times (I \to X)}
      	\arrow["{(P, \mooreCell)}", from=1-1, to=2-1]
      \end{tikzcd}\]
      \end{subfigure}%
      \hfill
      \begin{subfigure}[c]{0.5\textwidth}
        \scaletikzfig[0.8]{moore_cell2}
        \label{fig:moore_cell}
      \end{subfigure}%
      \caption{An algebra for $\Para(O \times (I \to -))$ as a parametric Moore cell. Here the output $O$ does not depend on the current input $I$.}
      \label{fig:moore_cell_algebra}
    \end{figure}
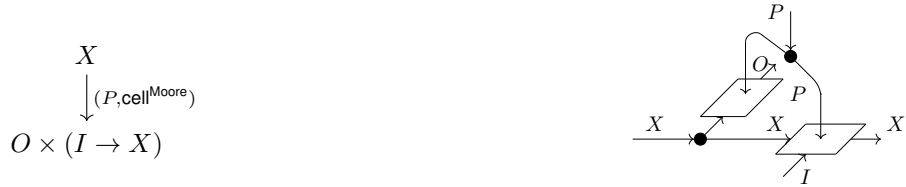
\end{example}

\section{Unrolling Neural Networks via Transfinite Construction}

%

We will now reap the benefits of the transfinite construction of free monads on endofunctors, and unpack the corresponding unrolling of these networks.
Specifically, with \cref{thm:parameters_lax_algebra_comonoids} hinting towards the path of having weight tying arise abstractly out of the categorical framework, we unpack the unrolling of these networks.
For the purposes of space and clarity, the unpaking is done as lax homomorphism of algebras for 2-endofunctors instead of for 2-monads.


\begin{example}[Iterated Folding RNN]
  \label{def:unrolled_recurrent}
  Consider an algebra $(P, \foldingRecurrentCell)$ of $\Para(1 + A \times -)$, i.e.\ a folding RNN cell (\cref{ex:folding_rnn_cell}).
  Its unrolling is a $\Para(1 + A \times -)$-algebra homomorphism below, where $\Delta_P : P \to P \times P$ is the copy map.
  \begin{equation}
    \label{eq:iterated_folding_rnn}
\begin{tikzcd}[ampersand replacement=\&]
	{1 + A \times \List{A}} \&\&\&\& {1 + A \times X} \\
	{\List{A}} \&\&\&\& X
	\arrow["{\gamma(\coproduct{\Nil}{\Cons})}"', from=1-1, to=2-1]
	\arrow["{(P, \foldingRecurrentCell)}", from=1-5, to=2-5]
	\arrow["{\Para(1 + A \times  -)((P, f_{\recurrent}))}", from=1-1, to=1-5]
	\arrow["{\Delta_P}"', shorten <=24pt, shorten >=37pt, Rightarrow, from=1-5, to=2-1]
	\arrow["{(P, f_{\recurrent})}"', from=2-1, to=2-5]
\end{tikzcd}
\end{equation}
Here the map $f_{\recurrent}$ is the parametric analogue of a \emph{fold} (\cref{ex:lists_initial}):

\begin{code}
@$f_{\recurrent}$@ :: (p, List a) -> x
@$f_{\recurrent}$@ p Nil = @$\foldingRecurrentCell$@ p (inl ())
@$f_{\recurrent}$@ p (Cons a as) = @$\foldingRecurrentCell$@ p (inr a (@$f_{\recurrent}$@ p as))
\end{code}

We can see that $f_{\recurrent}$ is structurally recursive: it processes the head of the list by applying $\foldingRecurrentCell$ to the parameter $p$ and the output of the $f_{\recurrent}$ with the same parameter, and the tail of the same list.
\begin{figure}[H]
  \scaletikzfig[0.7]{folding_rnn_unrolled}
  \caption{String diagram representation of the unrolled ``folding'' recurrent neural network.}
  \label{fig:iterated_folding_rnn}
\end{figure}
We proceed to show that $\Delta_P$ is indeed a valid reparameterisation for these parametric morphisms.
Unpacking the composition on top right side, and the bottom left side of \cref{eq:iterated_folding_rnn}, respectively, yields parametric maps
\begin{minipage}[t]{.5\textwidth}
    \[
P \times P \times (1 + A \times \List{A}) \to X
    \]
\end{minipage}%
\begin{minipage}[t]{0.5\textwidth}
    \[
    P \times (1 + A \times \List{A}) \to X    
    \]
\end{minipage}

whose implementations are, respectively:

\begin{minipage}[t]{.5\textwidth}
\begin{align*}
  (p, q, \inl(\bullet)) &\mapsto \foldingRecurrentCell(p, \inl(\bullet))\\
  (p, q, \inr(a, as)) &\mapsto \foldingRecurrentCell(q, a, f(p, as))
\end{align*}
\end{minipage}%
\begin{minipage}[t]{0.5\textwidth}
\begin{align*}
  (p, \bullet) &\mapsto \foldingRecurrentCell(p, \inl(\bullet))\\
  (p, (a, as)) &\mapsto \foldingRecurrentCell(p, a, f(p, as))
\end{align*}
\end{minipage}

It is easy to see that $\Delta_P$ is a valid reparameterisation between them, ``tying'' the parameters $p$ and $q$ together.
\end{example}

\begin{example}[Iterated unfolding RNN]
  \label{ex:iterated_unfolding_rnn}
  Consider a coalgebra $(P, \unfoldingRecurrentCell)$ of $\Para(O \times -)$, i.e.\ an unfolding RNN cell (\cref{ex:unfolding_rnn_cell}).
  Its unrolling is a $\Para(O \times -)$-coalgebra homomorphism below:
  \begin{equation}
    \label{eq:iterated_unfolding_rnn}
\begin{tikzcd}[ampersand replacement=\&]
	X \&\&\& {\Stream{O}} \\
	{O \times X} \&\&\& {O \times \Stream{O}}
	\arrow["{(P, \product{\cell_o}{\cell_n})}"', from=1-1, to=2-1]
	\arrow["{(P, f_{o, n})}", from=1-1, to=1-4]
	\arrow["{\Para(O \times -)((P, f_{o, n}))}"', from=2-1, to=2-4]
	\arrow["{\gamma(\product{\StreamOutput}{\StreamNext})}", from=1-4, to=2-4]
	\arrow["{\Delta_P}", shorten <=17pt, shorten >=17pt, Rightarrow, from=2-1, to=1-4]
\end{tikzcd}
\end{equation}

Here $f_{o, n} : P \times X \to \Stream{O}$ is the parametric analogue of an \emph{unfold} (\cref{ex:streams_terminal}):
\begin{code}
@$f_{o, n}$@ :: (p, x) -> Stream o
@$f_{o, n}$@ p x = MKStream (@$\cell_o$@ p x) (@$f_{o, n}$@ p (@$\cell_n$@ x))
\end{code}

We can see that $f_{o, n}$ is structurally corecursive: it produces a stream whose head is $\cell_o(p, x)$, and whose tail is the stream produced by $f_{o, n}$ at $\cell_n(p, x)$.

\begin{figure}[H]
  \scaletikzfig[0.72]{unfolding_unrolled}
  \caption{String diagram representation of the unrolled ``unfolding'' recurrent neural network.}
  \label{fig:iterated_unfolding_rnn}
\end{figure}
\end{example}

\begin{example}[Iterated Recursive NN]
  \label{ex:iterated_recursive_cell}
  Consider an algebra $(P, \recursiveCell)$ of $\Para(A + (-)^2)$, i.e.\ a recursive neural network cell (\cref{ex:recursive_cell}).
  Its unrolling is a $\Para(A + (-)^2)$-algebra homomorphism below:

  \begin{equation}
    \label{eq:iterated_recursive_cell}
\begin{tikzcd}[ampersand replacement=\&]
	{A + \Tree{A}^2} \&\&\&\& {A + X^2} \\
	{\Tree{A}} \&\&\&\& X
	\arrow["\cong"', from=1-1, to=2-1]
	\arrow["{(P, \recursiveCell)}", from=1-5, to=2-5]
	\arrow["{\Para(A + (-)^2)((P, f_{\recursive}))}", from=1-1, to=1-5]
	\arrow["{\Delta_P}"', shorten <=23pt, shorten >=35pt, Rightarrow, from=1-5, to=2-1]
	\arrow["{(P, f_{\recursive})}"', from=2-1, to=2-5]
\end{tikzcd}
  \end{equation}

Here $f_{\recursive}$ is a function with the following implementation

\begin{code}
@$f_{\recursive}$@ :: (p, Tree a) -> x
@$f_{\recursive}$@ p (Leaf a) = @$\recursiveCell$@ p (inl a)
@$f_{\recursive}$@ p (Node l r) = @$\recursiveCell$@ p inr (f (@$f_{\recursive}$@ p l) (@$f_{\recursive}$@ p r)
\end{code}

This function too performs weight sharing, as the recursive call to subtrees is done with the same parameter $p$.
  \begin{figure}[H]
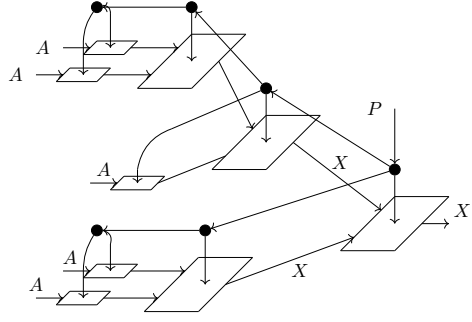

    \scaletikzfig[0.72]{recursive_cell_unrolled_lab-leaves_no-nodes}
    \caption{String diagram representation of the unrolled recursive neural network.}
    \label{fig:iterated_recursive_cell}
  \end{figure}
\end{example}

\begin{example}[Iterated Mealy machine cell]
  \label{ex:iterated_mealy_machine}
  Consider an algebra $(P, \mealyCell)$ of $\Para(I \to O \times -)$, i.e.\ a Mealy machine cell (\cref{ex:mealy_machine_as_coalgebra}).
  Its unrolling is a $\Para(I \to O \times -)$-algebra homomorphism below:
  \begin{equation}
    \label{eq:iterated_mealy_machine}
\begin{tikzcd}[ampersand replacement=\&]
	X \&\&\& {\Mealy{O}{I}} \\
	{(I \to O \times X)} \&\&\& {(I \to O \times \Mealy{O}{I})}
	\arrow["{(P, n)}"', from=1-1, to=2-1]
	\arrow["{\gamma(\StreamNext)}", from=1-4, to=2-4]
	\arrow["{(P, f_n)}", from=1-1, to=1-4]
	\arrow["{\Para(I \to O \times -)((P, f_n))}"', from=2-1, to=2-4]
	\arrow["{\Delta_P}", shorten <=19pt, shorten >=19pt, Rightarrow, from=2-1, to=1-4]
\end{tikzcd}
  \end{equation}

Here $f_n$ is a corecursive function with the following implementation

\begin{code}
@f\_n@ :: (p, x) -> Mealy o i
@f\_n@ (p, x) = MkMealy $ \i -> let (o', x') = n p x i
                             in (o', @f\_n@ p x')
\end{code}

  \begin{figure}[H]
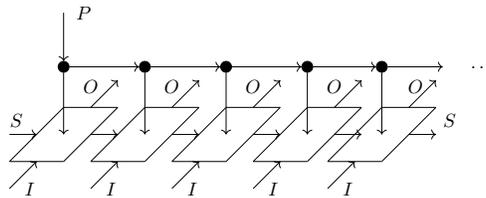

    \scaletikzfig[0.72]{mealy_unrolled}
    \caption{String diagram representation of a parametric Mealy machine --- a recurrent neural network.}
    \label{fig:mealy_unrolled}
  \end{figure}

\end{example}



\begin{example}[Iterated Moore machine cell]
  \label{ex:iterated_moore_machine}
  Consider a coalgebra $(P, \mooreCell)$ of $\Para(O \times (I \to -))$, i.e.\ a Moore machine cell (\cref{ex:moore_machine_as_coalgebra}).
  Its unrolling is a $\Para(O \times (I \to -))$-coalgebra homomorphism below:
  \begin{equation}
    \label{eq:iterated_moore_machine}
\begin{tikzcd}[ampersand replacement=\&]
	X \&\&\&\& {\Moore{O}{I}} \\
	{O \times (I \to X)} \&\&\&\& {(I \to O \times \Moore{O}{I})}
	\arrow["{(P, f_n)}", from=1-1, to=1-5]
	\arrow["{(P, \product{o}{n})}"', from=1-1, to=2-1]
	\arrow["{\gamma(\StreamNext)}", from=1-5, to=2-5]
	\arrow["{\Delta_P}", shorten <=25pt, shorten >=25pt, Rightarrow, from=2-1, to=1-5]
	\arrow["{\Para(O \times (I \to -))((P, f_n))}"', from=2-1, to=2-5]
\end{tikzcd}
\end{equation}

In code, $f_n$ is a corecursive function

\begin{code}
@f\_n@ :: (p, x) -> Moore o i
@f\_n@ (p, x) = MkMoore $ \i -> (o p x, @f\_n@ p (n p x i))
\end{code}

\begin{figure}[H]
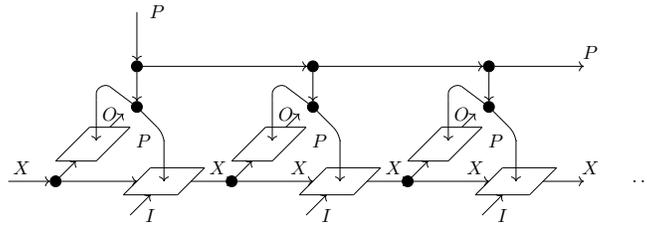

  \scaletikzfig[0.72]{moore_unrolled}
  \caption{String diagram representation of an unrolled Moore machine.}
  \label{fig:mealy_unrolled_2}
\end{figure}

%

\end{example}



\end{document}